\documentclass{article}
\RequirePackage{etex} %
\usepackage[letterpaper,margin=1in]{geometry}
\usepackage{amsmath,amssymb,amsthm,amsfonts,latexsym,bbm,xspace,graphicx,float,mathtools,mathdots}
\usepackage{braket,caption,ellipsis,xcolor,textcomp}
\usepackage{combelow} %
\usepackage{indentfirst}
\usepackage{booktabs}
\usepackage[hidelinks]{hyperref}
\usepackage[capitalize,noabbrev]{cleveref}
\crefname{ineq}{inequality}{inequalities}
\creflabelformat{ineq}{#2{\upshape(#1)}#3}

\usepackage{enumitem}

\newtheorem{theorem}{Theorem}

\usepackage{chngcntr}
\counterwithin{theorem}{section}

\newtheorem*{namedtheorem}{\theoremname}
\newcommand{\theoremname}{testing}

\newtheorem{lemma}[theorem]{Lemma}

\newtheorem{proposition}[theorem]{Proposition}

\theoremstyle{definition}

\newtheorem{remark}[theorem]{Remark}

\newcommand{\ignore}[1]{}

\usepackage{autonum}

\usepackage{algpseudocode}
\usepackage{placeins}
\usepackage{pgfplots}
\usepackage{graphicx}
\usepackage{subcaption}
\pgfplotsset{compat=newest}
\pgfplotsset{scaled y ticks=false}
\usepgfplotslibrary{groupplots}
\usepgfplotslibrary{dateplot}
\tikzstyle{every node}=[font=\small]
\pgfplotsset{
    yticklabel style={/pgf/number format/fixed},  
}

\pgfplotsset{compat=1.11,
 /pgfplots/ybar legend/.style={
 /pgfplots/legend image code/.code={
 \draw[##1,/tikz/.cd,yshift=-0.25em]
 (0cm,0cm) rectangle (3pt,0.8em);},
 },
}

\title{Efficient Sequential Evaluation of Large Language Models}
\usepackage{authblk}
\author{Chia-Yu Hsu}
\author{Shubhanshu Shekhar}
\affil{EECS Department, University of Michigan}
\affil{\texttt{\{chiayuh, shubhan\}@umich.edu}}

\begin{document}
\date{}
\maketitle

\begin{abstract}
We study the problem of sequentially evaluating a new large language model (LLM) on a fixed question set using historical performance data from prior LLMs. Our goal is to construct a confidence sequence (CS) for the model's capability on this question set and to design active querying rules that shrink the CS width as quickly as possible. For CS construction, we invert a family of test supermartingales and focus on two representative approaches: a reverse information projection (RIPr)-based approach and a testing-by-betting-based approach. We first study these approaches under an oracle setting, and demonstrate the oracle optimality of the RIPr-based construction. We then propose a growth-oriented querying rule that aims to maximize the worst-case one-step expected log-increment over the endpoints of the current CS.
In practice, we build these test supermartingales and the querying rule on predictions of question-level correctness learned from historical data. We then analyze the shrinkage behavior of the resulting CSs and identify two key factors that slow the shrinkage rate of CSs: accumulated prediction mismatch and the spikiness of the querying distribution. Finally, motivated by this analysis, we propose several mixture querying rules that combine growth-oriented querying, prediction refinement, and uniform exploration, trying to mitigate the effects that slow the shrinkage rate. We provide experiments comparing different querying rules for the RIPr-based and testing-by-betting-based CSs across several synthetic testing datasets. Interestingly, we observe that the simplest querying rule, uniform sampling, can sometimes outperform more adaptive querying rules for both methods.

\begin{center}
{\small\textit{This is a preliminary version; feedback is welcome.}}
\end{center}

\end{abstract}

\section{Introduction}\label{sec:intro}
Evaluating a newly developed large language model (LLM) on a benchmark consisting of a set of questions is a standard way to report the model's capability \cite{guo2023}. In this setting, the capability is often summarized by the average correctness rate of the LLM over the benchmark. However, modern benchmarks may contain thousands of questions or more (see \cite[Section 1]{wuCandes2026efficientevaluationllmperformance} for a review), making exhaustive or model-based evaluation time-consuming and costly. This naturally raises the following question: how can we query as few questions as possible while still obtaining an estimate of the average correctness with acceptable uncertainty?%

A recent work by \cite{wuCandes2026efficientevaluationllmperformance} studies this setting by leveraging historical data, namely the correctness records of previous LLMs on the same benchmark question set. These records are used to predict how likely the new model is to answer each question correctly, which in turn guides the evaluator to \emph{actively and adaptively} query informative questions. Specifically, they construct an AIPW-type mean accuracy estimator \cite{zrnic2026activestatisticalinference} with a querying rule that aims to both reduce the estimator variance and refine the prediction model. They analyze the resulting confidence interval (CI) and show two key properties. First, the CI is model-free in the sense that the validity of the CI does not rely on the accuracy of the prediction model. Second, the CI width decays asymptotically at the rate $O(1/\sqrt n)$, where $n$ is the total number of queried questions. They also empirically demonstrate that their proposed estimator outperforms several baselines, including uniform sampling.

However, in cost-sensitive LLM evaluation, the evaluator may wish to monitor the uncertainty after each query and stop once the estimate is sufficiently accurate, while still maintaining a valid uncertainty guarantee. Such sequential monitoring allows the number of queries to adapt to the strength of the evidence revealed by the data. A fixed-time confidence interval is not designed for this type of data-dependent stopping: its coverage guarantee holds only at a pre-specified number of queries and may fail when the stopping time is chosen adaptively based on the observed data. A more suitable tool is a \emph{confidence sequence} (CS) \cite{Darling1967ConfidenceSequences,WaudbySmith2024BoundedBetting,Howard2021ConfidenceSequences}, namely a sequence of confidence sets that maintains coverage uniformly over time, often referred to as anytime-valid coverage. Therefore, it is natural to ask the confidence-sequence counterpart of \cite{wuCandes2026efficientevaluationllmperformance}: Leveraging predictions learned from historical data, can we construct a model-free CS with an efficient querying rule? How fast does the CS width shrink?

\paragraph{Problem description.}
Let $N$ denote the total number of benchmark questions and write $[N]:=\{1,\ldots,N\}$ and $\Delta_N$ be the probability simplex over $[N]$. For each question $i\in[N]$, let $Z_i\in\{0,1\}$ indicate whether the new LLM answers question $i$ correctly. For analysis, we assume that each $Z_i\sim \mathrm{Ber}(p_i^\star)$ independently across question $i$, where $p_i^\star\in(0,1)$ for all $i\in[N]$, and denote the true parameter vector by $p^\star=(p_1^\star,\ldots,p_N^\star)$, and the target parameter $\theta^\star := \frac1N\sum_{i=1}^N p_i^\star$. We further assume access to historical correctness data from previous LLMs evaluated on the same question bank. %
In this paper, we use the Bayesian factor model proposed by \cite{wuCandes2026efficientevaluationllmperformance} to extract question-level predictions of the new model's correctness probabilities from historical data. The historical data consist of the responses of $M$ existing models on the same $N$-question benchmark. For simplicity, we assume that these historical responses are fully observed. We then focus on how to use the resulting predictions to jointly design the querying rule and an anytime-valid confidence sequence.

We aim to jointly design an active and adaptive querying rule $\{q_t\}_{t\ge1}$, where $q_t\in\Delta_N$, and a CS $\{C_t\}_{t\ge1}$, where $C_t\subseteq[0,1]$, for estimating $\theta^\star$. The CS is required to satisfy the time-uniform coverage guarantee
$\mathbb{P}\left(\exists\,t\ge1:\theta^\star\notin C_t\right)\leq \alpha$,
where $\alpha\in(0,1)$ is a pre-specified error tolerance. We refer to this as a level-$\alpha$ time-uniform coverage guarantee. Equivalently, for any stopping time $\tau\in\mathbb{N}$, the stopped confidence set remains valid: $\mathbb{P}\left(\theta^\star\notin C_\tau\right)\leq\alpha$.
Moreover, we seek querying rules for which the confidence sets shrink rapidly over time. Since $C_t$ need not be an interval, we define its width by
$|C_t|:=\sup C_t-\inf C_t$. %

\paragraph{Contributions.}
Our contributions are summarized as follows. First, in an oracle setting where the true question-level correctness probabilities are known, we study the general CS construction based on test supermartingales in \cite[Theorem~1]{WaudbySmith2024BoundedBetting} under active querying. We observe that there is generally no single querying rule that is optimal for the growth of all test supermartingales simultaneously. This motivates a natural surrogate: a growth-oriented maxmin criterion for choosing the querying distribution. We then instantiate this framework using two well-known e-value constructions, reverse information projection (RIPr) and testing-by-betting, to construct the corresponding test supermartingales and CSs under active sampling. The RIPr-based construction enjoys oracle growth optimality. Moreover, for the testing-by-betting-based approach, we propose a computationally lighter variant based on a quadratic lower bound, for which the resulting growth-oriented sampling rule coincides with the variance-reduction sampling rule in \cite{wuCandes2026efficientevaluationllmperformance}.   %

Second, we incorporate the predictions of question-level correctness learned from the historical data in the constructed RIPr-based and testing-by-betting-based approaches to replace the true correctness, resulting in fully data-dependent CSs and querying rules. Then, by analyzing the shrinkage rate of the plug-in CS width in this practical setting, we identify two major factors that can slow down CS shrinkage: accumulated prediction mismatch, quantified by a KL-divergence-type term for the RIPr-based approach and an MSE-type term for the testing-by-betting-based approach, as well as the spikiness of the querying distribution. Our analysis further suggests that prediction mismatch has a more pronounced effect on the RIPr-based approach, whereas the spikiness of the querying distribution affects the testing-by-betting-based approach more directly. 

Finally, we design several mixture querying rules that aim to balance growth-oriented querying, prediction refinement, and the reduction of spikiness in the querying distribution. We empirically compare these rules for both the RIPr-based and testing-by-betting-based CSs. Interestingly, we find that the pure uniform querying rule, despite its simplicity, can be a strong and competitive baseline on some datasets.

\paragraph{Paper organization.} The remainder of this paper is organized as follows. The literature review is provided in Section~\ref{sec:related_work}. Section~\ref{sec:prelim} introduces the test-supermartingale construction of CSs under active querying and the maxmin criterion for active selection. Section~\ref{sec:oracle} instantiates this framework for the RIPr-based and testing-by-betting-based approaches under the oracle setting. Section~\ref{sec:practical} develops the practical plug-in versions and analyzes the shrinkage behavior of the resulting CSs. Section~\ref{sec:experiment} empirically examines the oracle optimality of the RIPr-based approach, proposes several mixture querying rules for the practical setting, and compares their performance. Section~\ref{sec:conclusion} concludes the paper.

\section{Related work}\label{sec:related_work}
We review three lines of related work. First, we discuss standard confidence-sequence constructions and the e-process tools that motivate our RIPr-based and testing-by-betting methods. Second, we review anytime-valid inference under adaptive sampling. %
Third, we discuss prediction-assisted active statistical inference.

\paragraph{Confidence sequences and test supermartingale constructions.}
Standard confidence sequences for non-adaptive data, such as i.i.d. observations, have been widely studied \cite{Darling1967ConfidenceSequences,Howard2021ConfidenceSequences,WaudbySmith2024BoundedBetting,chugg2025timeuniformconfidencespheresmeans}. Classical concentration-based CSs provide explicit width guarantees but can be conservative. Testing-by-betting CSs offer a more adaptive alternative and have been shown to improve upon concentration-based constructions both theoretically and empirically by \cite[Section 4]{WaudbySmith2024BoundedBetting} and \cite{shekhar2023}, although their widths are typically non-explicit. From the perspective of maximizing the one-step growth rate, RIPr \cite{grunwald2023,ramdas2023gametheoreticstatisticssafeanytimevalid} provides a growth-rate-optimal factor when the relevant projection can be computed, and RIPr-based CSs have also been studied \cite{turner2022exactanytimevalidconfidenceintervals}. These works motivate our use of RIPr and testing by betting as complementary CS constructions: RIPr provides an optimal growth-rate (in some sense) benchmark, while testing by betting offers a practically used alternative. However, they do not directly address our problem, where we also need to design an efficient adaptive querying procedure. %

\paragraph{Anytime-valid inference under adaptive sampling.}
Beyond the standard i.i.d. setting, confidence sequences have also been studied under adaptive data collection. The closest work to ours is \cite{shekhar23aRiskLimitingFinancialAudits}, which studies weighted mean estimation under adaptive, data-dependent sampling and explicitly designs without-replacement sampling distributions to reduce the width of valid betting-based CSs. Their method, however, relies on pre-specified weights for the data points, whereas in our setting, we consider estimating the average correctness without any prior weight information. Therefore, their proposed sampling approach cannot be applied to our problem.

Other related work on online and adaptive experiments studies anytime-valid inference under bandit-style treatment allocation \cite{ham2023designbased}, adaptive average-treatment-effect inference \cite{cook2024semiparametricefficientinferenceadaptive}, contextual-bandit off-policy inference \cite{waudbysmith2024anytimevalidoffpolicyinferencecontextual}, and best-model identification for LLMs \cite{tolochinsky2026validbestmodelidentificationllm}. These works ensure validity under adaptive sampling, but do not focus on designing active queries to shrink a CS for the average correctness of a single LLM over a fixed benchmark.

\paragraph{Prediction-assisted active statistical inference.}
\cite{zrnic2026activestatisticalinference} and related work on prediction-powered inference \cite{angelopoulos2023predictionpoweredinference} (see \cite[Section 3]{zrnic2026activestatisticalinference} for more references) consider a stream of test samples with their unobserved labels, and the goal is to estimate ``a parameter of the expected loss" of a supervised learning model by using model predictions together with a limited number of queried labels and provide a CI for their estimator with an asymptotic width. Moreover, their estimator can also be translated to a CS by the betting approach \cite[Theorem 3]{waudbysmith2024anytimevalidoffpolicyinferencecontextual}. They also utilize the AIPW estimator, but the key difference with our paper is in the notion of activeness. In their setting, samples arrive in a fixed natural order, and the active decision is whether to query the true label of the current sample. In our setting, the evaluator actively chooses which question to query from a fixed benchmark. This distinction matters because the natural order of the stream may be statistically inefficient for estimating the average correctness over the whole question bank; for example, if the stream begins with many easy questions, the CS may shrink slowly for the target population mean. \cite{sfyraki2026} revisits this problem, providing a non-asymptotic analysis of an anytime-valid width for the estimator proposed by \cite{zrnic2026activestatisticalinference} and argues that uniform sampling may be competitive. We also discover a similar result in our active querying setting. 

\begin{remark}
    Due to space limitations, we do not provide a separate review of the broader LLM evaluation literature. Our setting is closely related to \cite{wuCandes2026efficientevaluationllmperformance}, but focuses on providing anytime-valid uncertainty guarantees for sequential evaluation. Thus, our position within the LLM evaluation literature is similar to theirs. Readers interested in a broader discussion of LLM evaluation are referred to \cite[Section 2]{wuCandes2026efficientevaluationllmperformance}.
\end{remark}

\section{Test-supermartingale-based CSs with active selection}\label{sec:prelim}

We begin by introducing the general test-supermartingale-based construction of CSs under our active querying setting.
\paragraph{Setting.} We consider a predictable querying rule $\{q_t\}_{t\ge1}$, where $q_t\in\Delta_N$ is $\mathcal F_{t-1}$-measurable. Here, $\mathcal F_t:=\sigma(\{(I_s,Z_{I_s})\}_{s=1}^t)$ is the filtration generated by the queried questions and their observed correctness outcomes up to time $t$, $I_s\in[N]$ denotes the question queried at round $s$, and $\mathcal F_0:=\{\emptyset,\Omega\}$. Conditioned on $\mathcal F_{t-1}$, we let $Q_{p^\star,q_t}$ denote the joint distribution of one active-query observation $(I_t,Z_{I_t})$, defined by $Q_{p^\star,q_t}(i,z):=q_t(i)(p_i^\star)^z(1-p_i^\star)^{1-z}$ for all $i\in[N]$ and $z\in\{0,1\}$.

We mainly follow the construction proposed by \cite[Theorem 1]{WaudbySmith2024BoundedBetting} to build a valid CS through test supermartingales: A level-$\alpha$ CS can be universally constructed by a family of \emph{test supermartingales}, denoted as $\{W_t(m)\}_{t\ge0}$ for all $m\in[0,1]$. Intuitively, the test supermartingale for the candidate value $m$ aims to accumulate evidence against this value if $m\neq \theta^\star$, while remaining low if $m=\theta^\star$. The CS is then obtained by retaining the candidate values whose test supermartingales have not crossed the rejection threshold $1/\alpha$. Technically, for each candidate value $m\in[0,1]$, define the composite null class $\mathcal P_m
:= \left\{
p\in[0,1]^N:
\frac1N\sum_{i=1}^N p_i=m
\right\}$. 
For each $m$, if there are one-step factors $\{E_t(m)\}_{t\ge1}$ satisfying the \emph{conditional e-value} property: $\forall\,t\ge1$, $E_t(m)\ge0$, $\mathbb{E}_{(I,Z_I)\sim Q_{p,q_t}}\left[
E_t(m)\mid\mathcal F_{t-1}
\right]\le1$, $\forall p\in\mathcal P_m$.
Then, we construct the process $\{W_t(m)\}_{t\ge0}$, $W_t(m):=\prod_{s=1}^t E_s(m)$ with $W_0(m)=1$, which satisfies $W_t(m)\ge0$ and $ \mathbb{E}_{(I,Z_I)\sim Q_{p,q_t}}
\left[
W_t(m)\mid\mathcal F_{t-1}
\right]\leq W_{t-1}(m)$ for each $t\ge1$ and $m\in[0,1]$ for all $p\in\mathcal{P}_m$ (see Lemma~\ref{lem:product_of_e-val_is_e-process} for the proof). This process is defined as a \emph{test supermartingale} for the composite null $\mathcal P_m$. Consequently, by \cite[Theorem 1]{WaudbySmith2024BoundedBetting} (based on Ville's inequality), a level-$\alpha$ confidence sequence is given by $C_t
:=\{
m\in[0,1]:
W_t(m)<1/\alpha\}$, $\forall\,t\ge1$.
Therefore, under this construction, the major goal is to design
\begin{enumerate}
    \item The conditional e-value $\{E_t(m)\}_{t\ge1}$ under the null class $\mathcal{P}_m$ for each $m\in[0,1]$.
    \item The (adaptive) predictable querying policy $\{q_t\}_{t\ge1}$ for making $E_t(m)$ larger for false candidates $m\neq\theta^\star$. 
\end{enumerate}
Specifically, we consider maximizing the expected log-increment
\[
G_t(q_t,E_t(m),p^\star)
:=
\mathbb{E}_{(I,Z_I)\sim Q_{p^\star,q_t}}
\left[
\log E_t(m)\mid \mathcal F_{t-1}
\right],
\qquad m\neq \theta^\star,
\]
which is also known as the Kelly criterion \cite{kelly1956new,BREIMAN11}, as a principle for designing the querying rule. Under this criterion, we observe that the growth-optimal querying distribution generally depends on the candidate value $m$. Since our goal is to shrink the width of the confidence sequence, a natural surrogate is to focus on the two boundary points of the current CS. Therefore, our design is to choose $q_t$ to maximize the lower evidence-growth rate between these two endpoints:
\begin{equation}
q_t\in
\arg\max_{q\in\Delta_N}
\min\left\{
G_t(q,E_t(U_{t-1}),p^\star),
G_t(q,E_t(L_{t-1}),p^\star)
\right\},\label{eq:oracle_maxmin_sampling}
\end{equation}
where $L_t:=\inf C_t$ and $U_t:=\sup C_t$. We refer to this maxmin selection criterion as \emph{growth-oriented} querying. We remark that, in Section~\ref{sec:exp_mixture}, we adapt this growth-oriented rule to the practical setting by replacing $p^\star$ with its prediction and mixing the resulting rule with other querying strategies.

In the next section, we first assume $p^\star$ is known and design the conditional e-values for the RIPr-based and testing-by-betting-based approaches. This oracle setting simplifies the discussion and clarifies how the question-level correctness vector is used in the construction, which will later be replaced by its prediction in the practical setting in Section~\ref{sec:prac_e-values}.

\section{Oracle conditional e-value construction}\label{sec:oracle}

In this section, we follow the setting defined in Section~\ref{sec:prelim} and assume $p^\star$ is known. %
\subsection{RIPr-based approach}\label{sec:ripr_based_oracle_design}

The basic idea of RIPr is to identify the reverse information projection (or RIPr point) of the alternative distribution onto the composite null class and use the resulting density ratio to construct a valid e-value. RIPr point is the least favorable distribution in this null class, namely the one closest to the alternative in KL divergence. Generally, the likelihood ratio between the reference alternative and this RIPr point is then used as the e-value. This choice is principled because, among all valid e-values for the composite null, it maximizes the expected log-increment under the chosen alternative (see \cite[Chapter 6]{Ramdas_2025} for more details).

Therefore, in the oracle setting, it is natural to use the true Bernoulli parameter vector $p^\star$ as the reference alternative, since it represents the ideal distribution against which evidence for rejecting false candidates should be accumulated.
 Then, we define the conditional version of the RIPr point under a sampling rule $\{q_t\}_{t\ge1}$ for $m\in[0,1]$ as, for each $t\ge1$, given $\mathcal{F}_{t-1}$,
\[
p_{q_t}^{\dagger}(m;p^\star)
\in
\arg\min_{p\in\mathcal P_m}
D_{\mathrm{KL}}\left(Q_{p^\star,q_t}\,\Vert\,Q_{p,q_t}\right)=
\arg\min_{p\in\mathcal P_m}
\sum_{i=1}^N q_t(i)d_{\mathrm{kl}}(p_i^\star,p_i),
\]
where $d_{\mathrm{kl}}(\cdot)$ denotes the binary KL divergence. %
Through some mild conditions, $p^\dagger_{q_t}(m;p^\star)$ can be solved by a 1-dimensional root-finding algorithm such as the bisection method (see Proposition~\ref{prop:solve_ripr_point}). Then, we have the following proposition, which is proved in Appendix~\ref{proof:proof_of_exp_log-wealth_opt_fix_query_strategy} for completeness (It has been proved at a general level by \cite{larsson2025}).
\begin{proposition}[RIPr conditional e-value under a fixed active querying strategy]\label{prop:ripr_validity}
Fix any candidate value $m\in(0,1)$ and any predictable querying rule $q_t(i)>0$ for all $t\ge1$ and $i\in[N]$. Suppose $p^\star_i\in(0,1)$ for all $i\in[N]$. 
The oracle RIPr one-step factor $E_t^{\mathrm{RIPr}}(m, p^\star, q_t)
:=
\frac{
(p_{I_t}^\star)^{Z_{I_t}}(1-p_{I_t}^\star)^{1-Z_{I_t}}
}{
(p_{q_t}^\dagger(m;p^\star)_{I_t})^{Z_{I_t}}(1-p_{q_t}^\dagger(m;p^\star)_{I_t})^{1-Z_{I_t}}
}$ is a conditional e-value for $\mathcal{P}_m$ for each $m\in(0,1)$, where $p_{q_t}^\dagger(m;p^\star)_{I_t}$ is the $I_t$-th element of $p_{q_t}^\dagger(m;p^\star)$. For $m=0$ or $m=1$; see Appendix~\ref{proof:proof_of_exp_log-wealth_opt_fix_query_strategy}.
\end{proposition}
Then, by the fact of the RIPr e-value \cite[Chapter 6]{Ramdas_2025}, we know that the expected log-increment of the RIPr e-value satisfies,
\begin{equation}
G_t(q_t,E_t^{\mathrm{RIPr}}(m, p^\star, q_t),p^\star)
\geq
G_t(q_t,E_t(m),p^\star)
\qquad\forall E_t(m)\in\mathcal E_m(q_t)\qquad \text{a.s.},  \label{eq:ripr_oracle_optimality}  
\end{equation}
where for each $t\geq1$, $\mathcal E_m(q_t)$ denotes the class of conditional e-values for $\mathcal P_m$ with the predictable sampling distribution $\{q_t\}_{t\ge1}$.
This fact clarifies the role of the RIPr approach: under any fixed predictable querying rule (with a positive lower bound of the sampling probability), RIPr maximizes the oracle expected log-increment among valid conditional e-values.

However, constructing the RIPr e-value can be computationally expensive, since it requires solving a root-finding problem for each candidate value. We next introduce the testing-by-betting approach, which yields a much simpler conditional e-value construction in some cases. %

\subsection{Testing-by-betting-based approach}\label{sec:testing_by_betting_based_oracle_design}
The idea of the betting-based approach \cite[Section~4]{WaudbySmith2024BoundedBetting} is to accumulate evidence against the null through a betting game. At each round, for every candidate value $m\in[0,1]$, a predictable bet $\lambda_t(m)$ is placed on an estimator $\Phi_t$ of the target parameter $\theta^\star$, and these two components form a conditional e-value. Ideally, the resulting wealth process grows rapidly when $m\neq\theta^\star$ while remaining controlled when $m=\theta^\star$. %

We use the one-step estimator proposed by \cite{wuCandes2026efficientevaluationllmperformance}: $\Phi_t(p,q_t) := \frac1N\sum_{i=1}^N p_i+\frac1N\frac{Z_{I_t}-p_{I_t}}{q_t(I_t)} $, where $p\in\mathcal{P}_m$ for $m\in[0,1]$, and have the following proposition, which is proved in Appendix~\ref{proof:test_by_bet_validity}.
\begin{proposition}[Validity of testing-by-betting approach]\label{prop:test_by_bet_validity}
    Fix any candidate value $m\in[0,1]$ and any predictable querying rule $q_t(i)>0$ for all $t\ge1$ and $i\in[N]$. Let $h_t:=\min_{i\in[N]}q_t(i)$ for each $t\ge1$ and $\Lambda_{h_t}(m) :=\left[-\frac1{1-m+\frac1{Nh_t}}, \frac1{m+\frac1{Nh_t}}\right]$ for any $m\in[0,1]$. Then, for any $t\ge1$ and $m\in[0,1]$, if $\lambda_t(m)\in\Lambda_{h_t}(m)$, the testing-by-betting one-step factor $E_t^{\mathrm{bet}}(m,\lambda_t(m,q_t),p^\star,q_t) := 1+\lambda_t(m,q_t)\bigl(\Phi_t(p^\star,q_t)-m\bigr)$ is a conditional e-value for $\mathcal{P}_m$ for each $m\in[0,1]$.
\end{proposition}
Then, assuming we already have a sampling rule, the next natural question is how to design the bet. Notably, the valid betting range depends explicitly on the querying distribution through the $h_t$. Thus, the choice of the querying distribution and the efficiency of evidence growth are tightly coupled. In words, if the querying rule assigns very small probabilities to some questions, then $h_t$ becomes small, making the valid betting range shrink. This forces the bet to be conservative, which can substantially reduce the expected log-increment and slow down the growth of the wealth process (we provide a theoretical understanding of this argument in Theorem~\ref{thm:mismatch_impact_for_bet}). Then, a general approach for determining the bet is: first choose the $q_t$ according to a rule. Then, given this $q_t$, choose the bet that maximizes the expected log-increment of each candidate value, over the valid betting range.

Although finding the optimal bet may still require a bisection procedure, the testing-by-betting construction has a simpler structure than the RIPr-based e-value. This simpler form allows us to derive a tractable approximation for choosing a surrogate betting parameter under a given querying rule $\{q_t\}_{t\ge1}$.

We proceed to this alternative based on the quadratic lower bound $\log(1+x)\ge x-x^2$ for $x\ge -1/2$. For each $m\in[0,1]$, we restrict the bet to the smaller range $\frac12\Lambda_{h_t}(m)$, obtained by multiplying both endpoints of $\Lambda_{h_t}(m)$ by $1/2$. This restriction ensures that
$\lambda_t(m)\bigl(\Phi_t(p^\star,q_t)-m\bigr)\ge -\frac12$
for all possible observations. Therefore, instead of maximizing the bet on the exact expected log-increment, we maximize its quadratic lower bound
\begin{equation}
    \mathbb E_{Q_{p^\star,q_t}}
\left[
\lambda_t(m)(\Phi_t(p^\star,q_t)-m)-(\lambda_t(m))^2(\Phi_t(p^\star,q_t)-m)^2
\,\middle|\,
\mathcal F_{t-1}
\right].\label{eq:exp_quadratic_lower_bound}
\end{equation}
Since the quadratic lower bound is concave in $\lambda_t(m)$, the optimized bet admits an explicit form for any fixed $q_t$. We provide the details in Proposition~\ref{thm:quadratic_sampling_rule_and_bet}. Interestingly, in some cases, substituting this optimized bet back into \eqref{eq:exp_quadratic_lower_bound} yields the sampling distribution $q_t(i)\propto \sqrt{p_i^\star(1-p_i^\star)}$ for each $i\in[N]$ for maximizing the bound, which coincides with the variance-reduction sampling rule proposed by \cite{wuCandes2026efficientevaluationllmperformance}.

The above constructions of conditional e-values require the $p^\star$, which is unavailable in practice. In the next section, we replace the role of $p^\star$ with predictions of question-level correctness obtained from a Bayesian factor model used in \cite{wuCandes2026efficientevaluationllmperformance} fitted using historical data.

\section{Practical setting}\label{sec:practical}
\subsection{Practical conditional e-values}\label{sec:prac_e-values}
We denote the question-level prediction at round $t$ by $\hat p_t:=\{\hat p_t(1),\ldots,\hat p_t(N)\}$, where $\hat p_t(i)$ represents the predicted probability that the new model answers question $i$ correctly. The prediction is initialized using the model and question embeddings fitted from historical correctness data. After each query $t$ and obtaining $(I_t, Z_{I_t})$, we update $\hat p_t$ using the online Laplace update by this drawn tuple, $\mathcal{F}_{t}$. We omit the detailed update equations and refer the reader to \cite{wuCandes2026efficientevaluationllmperformance}. Specifically, the $\hat p_t(j)$ here is equivalent to the $\hat p^{(t)}_j$ in their paper.

For establishing the practical conditional e-values, we simply replace $p^\star$ with $\hat p_{t-1}$ at each round $t\ge1$ in the constructions of $E_t^{\mathrm{RIPr}}$ and $E_t^{\mathrm{bet}}$ above. Since $\hat p_{t-1}$ is $\mathcal{F}_{t-1}$-measurable, one can easily verify that this plug-in procedure preserves the validity in Proposition~\ref{prop:ripr_validity} and Proposition~\ref{prop:test_by_bet_validity} (hence, other predictable predictions are also valid to be used). Moreover, we also apply this plug-in method in proposing the growth-oriented sampling in \eqref{eq:oracle_maxmin_sampling} by replacing $p^\star$ by $\hat p_{t-1}$ at each $t\ge1$ in computing the expectation and e-values. We omit the explicit expressions. %

However, this plug-in method introduces a mismatch between the prediction $\hat p_{t-1}$ and the true correctness vector $p^\star$ at each round $t$. In Section~\ref{sec:impact_of_mismatch_and_variance}, we discuss how this mismatch can degrade the shrinkage behavior of the resulting CS and identify other key factors that affect this shrinkage. This analysis shows that the plug-in maxmin selection rule may not be the only consideration in practical querying-rule design.

\subsection{Analysis of the shrinkage behavior of CSs}\label{sec:impact_of_mismatch_and_variance}
We apply Freedman's inequality \cite{freedman1975tail} to provide a time-uniform high-probability upper bound for an ``approximated" width of the CS. In the following, the considered conditional e-value $E^{\mathrm{RIPr}}_t(\cdot)$ and $E^{\mathrm{bet}}_t(\cdot)$ are the plug-in versions mentioned in Section~\ref{sec:prac_e-values}.
\begin{theorem}[(Informal) Approximated width of RIPr approach]\label{thm:mismatch_impact_for_RIPr}
For any $t\ge1$, suppose the $\log E_t^{\mathrm{RIPr}}(\cdot)$ is bounded and the querying rule satisfies $q_t(i)\geq h$ for some $h\in(0,1/N]$ for any $i\in[N]$. With probability at least $1-\delta$, simultaneously for all $t\le N$,
\[
|C_t^{\mathrm{RIPr}}\cap\mathcal{M}_K|\in O\left(\sqrt{
\frac{
\log(1/\alpha)
+
\sigma^{\mathrm{RIPr}}_t\sqrt{\log(K\log N/\delta)}
+
\mathcal \sum_{s=1}^t
D\!\left(
Q_{p^\star,q_s}\,\Vert\,Q_{\hat p_{s-1},q_s}
\right)
}{
tNh
}
}\right),
\]
where $\mathcal{M}_K$ is the set of $K$ grids and $\sigma^{\mathrm{RIPr}}_t:=\sqrt{\sup_{m\in\mathcal{M}_K}\sum_{s=1}^{t}\mathrm{Var}(\log E_s^{\mathrm{RIPr}}(\cdot)\,|\,\mathcal{F}_{t-1})}$.
\end{theorem}
The proof of Theorem~\ref{thm:mismatch_impact_for_RIPr} is proved in Appendix~\ref{proof:proof_of_mismatch_impact_for_RIPr}. Similarly, for the testing-by-betting approach,  we also have the following theorem, which is proved in Appendix~\ref{proof:proof_of_mismatch_impact_for_bet}.
\begin{theorem}[(Informal) Approximated width of testing-by-betting approach]\label{thm:mismatch_impact_for_bet}
    For any $t\ge1$, suppose the $\log E_t^{\mathrm{bet}}(\cdot)$ is bounded, the querying rule satisfies $q_t(i)\geq h$ for some $h\in(0,1/N]$ for any $i\in[N]$, and the bet $\lambda_t(m)$ satisfies some oracle regularized conditions for all $m\in\mathcal{M}_K$.
    With probability at least $1-\delta$, simultaneously for all $t\le N$, 
\[
    |C_t^{\mathrm{bet}}\cap \mathcal{M}_{K}|\in O\left(\sqrt{\frac{\log(1/\alpha)+\sigma^{\mathrm{bet}}_t\sqrt{\log(K\log N/\delta)}\left(2+\overline{\mathrm{MSE}}_t\right)}{tNh}}\right)
\]
where $\sigma^{\mathrm{bet}}_t:=\sqrt{\sup_{m\in\mathcal{M}_K}\sum_{s=1}^{t}\mathrm{Var}(\log E_s^{\mathrm{bet}}(\cdot)\,|\,\mathcal{F}_{t-1})}$, and $\overline{\mathrm{MSE}}_t:=\frac{1}{t}\sum_{s=1}^{t}\frac{1}{N}\sum_{i=1}^{N}(p^\star(i)-\hat p_{s-1}(i))^2$.
\end{theorem}

We first observe that prediction mismatch is usually more pronounced for the RIPr-based CS than for the betting-based CS; see Appendix~\ref{sec:exp_verify_plug-in_impact} for an empirical verification. This difference comes from the structure of the two conditional e-values. In the RIPr-based CS, the plug-in prediction forms the likelihood-ratio-type log-increment, so an inaccurate prediction directly reduces the conditional expected log-increment and leads to an accumulated KL mismatch term. In contrast, in the betting-based CS, the plug-in prediction enters through an AIPW-type estimator $\Phi_t(\cdot)$ that remains conditionally unbiased for $\theta^\star$. Hence, prediction error mainly affects the betting-based CS through the variance or second-moment term, rather than directly biasing the growth direction. This makes the betting-based CS relatively more robust to prediction mismatch.

However, the betting-based CS is more sensitive to the lower bound $h$ of the querying distribution. This is because if $h$ is small, the valid betting range must be correspondingly restricted to keep the conditional betting-based e-value nonnegative; see Proposition~\ref{prop:test_by_bet_validity}. This dependence is reflected in Theorem~\ref{thm:mismatch_impact_for_bet} through the term $tNh$ in the denominator. Even when the prediction error is small, a small value of $h$ weakens the betting growth and leads to a wider CS width; see Section~\ref{sec:exp_verify_oracle} for an empirical verification. By contrast, the RIPr-based bound also contains $h$, but this dependence mainly comes from a loose lower bound on the accumulated RIPr log-increment rather than from an explicit importance-weighted correction; see the analysis in Appendix~\ref{proof:proof_of_mismatch_impact_for_RIPr} and Lemma~\ref{lem:oracle_ripr_curvature}. %

Therefore, in practice, there may not be a universally best approach, and the preferable choice can depend on the dataset and the quality of the prediction. For querying-rule design, one should pay attention to both the lower bound of the querying distribution and the mismatch behavior of the plug-in prediction. A natural approach is to use a mixture querying rule that balances growth-oriented sampling, prediction refinement, and a part of uniform sampling. We provide several querying rules in Section~\ref{sec:exp_mixture} and conclude this discussion with the following remarks.

\begin{remark} 
    These two theorems provide an initial theoretical understanding of how certain aspects of the sampling rule affect the CS width. However, they do not fully characterize the optimal design of the sampling rule. In particular, the bounds hold for any sampling rule satisfying $q_t(i)\ge h$ for all $i\in[N]$, and therefore mainly capture the effect of the minimum sampling probability. They do not explicitly reflect finer properties of the sampling rule, such as the maxmin growth criterion discussed above or the detailed structure of the querying distribution. Moreover, these bounds do not fully control the accumulated variance term. Consequently, they can be suboptimal in order and are too coarse to provide a sharp comparison between the RIPr-based and testing-by-betting-based methods. We therefore interpret them only as \emph{diagnostic bounds}: they highlight how prediction mismatch and the lower bound on the querying probabilities $q_t(i)$ may influence the CS shrinkage. The actual comparison between the two methods is primarily supported by the experiments, where we observe qualitative trends consistent with these factors. We still work on making these bounds tighter.
\end{remark}

\section{Experiments}\label{sec:experiment}
In this section, we first empirically examine the oracle setting for both the RIPr- and the betting-based CSs described in Section~\ref{sec:oracle}. To focus on the performance difference induced by their e-values, we consider only two querying rules in the oracle experiments: the growth-oriented sampling rule described in \eqref{eq:oracle_maxmin_sampling} and uniform sampling as a benchmark. We then move to the practical setting described in Section~\ref{sec:practical} to compare both methods, along with several sampling rules. Notably, for the betting-based CS, the implementation-light variant is not evaluated separately. It uses a smaller betting range than the original method, so its statistical performance is expected to be weaker; its main advantage is lower computational cost. Moreover, the anytime coverage is demonstrated in  Appendix~\ref{append:experiment_details_oracle} Appendix~\ref{append:experiment_details_practical}, which are evaluated with respect to the realized benchmark accuracy induced by each sampled correctness vector.

\subsection{Oracle RIPr-based CS vs. Oracle betting-based CS}\label{sec:exp_verify_oracle}
We compare these two oracle approaches in terms of the stopping time, the first time that the CS width falls below a redefined accuracy $\epsilon$, in two regimes of $p^\star$. The result is shown in Figure~\ref{fig:oracle_avg_stopping_time}. 
These empirical results lead to the following observations. First, the testing-by-betting-based CS is highly sensitive to the lower bound of the querying distribution. The testing-by-betting-based CS with growth-oriented sampling (green) performs much worse than the same CS with uniform sampling (red). This supports the sensitivity discussed in Sections~\ref{sec:testing_by_betting_based_oracle_design} and~\ref{sec:impact_of_mismatch_and_variance}, even in the oracle setting where $p^\star$ is known. Second, both the RIPr-based CS with growth-oriented sampling (purple) and the RIPr-based CS with uniform sampling (yellow) outperform the corresponding testing-by-betting-based CSs in most cases. This is consistent with the oracle growth optimality of the RIPr construction discussed in Section~\ref{sec:ripr_based_oracle_design}. Third, uniform sampling is a surprisingly competitive baseline for both CS approaches. This behavior is similar to \cite{sfyraki2026}, where they consider a prediction-powered active inference with relatively passive sampling. In the following experiments on practical CSs, we also observe that uniform sampling remains competitive across different synthetic benchmark problem settings.

\begin{figure}[t]
    \centering
    \includegraphics[width=0.9\textwidth]{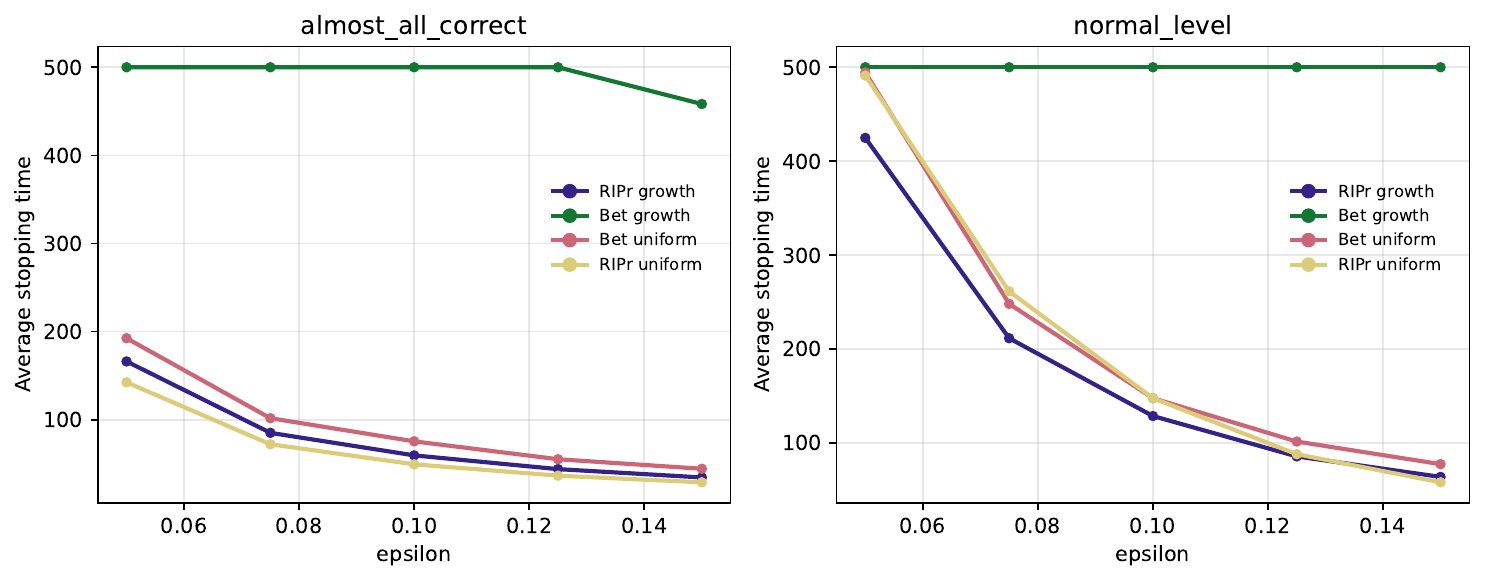}
    \caption{Average stopping time versus the prescribed accuracy levels $\epsilon\in\{0.05,0.075,0.1,0.125,0.15\}$, with a total of $N=500$ questions. Smaller values indicate higher sampling efficiency, and curves reaching $500$ correspond to runs that did not attain the target width within the query budget. \texttt{almost\_all\_correct} and \texttt{normal\_level} represent different settings of $p^\star$. Due to space limitations, the detailed setting of this experiment is provided in Appendix~\ref{append:experiment_details_oracle}.}
    \label{fig:oracle_avg_stopping_time}
\end{figure}

\subsection{Practical RIPr-based CS vs. Practical Betting-based CS}\label{sec:exp_mixture}

In this experiment, we use the plug-in e-values described in Section~\ref{sec:prac_e-values}. We compare the following selection rules with the two approaches: growth-oriented sampling, uniform sampling, a predesigned time-scheduled mixture, and a data-dependent mixture. 

The growth-oriented one and the uniform one follow the oracle experiments in Section~\ref{sec:exp_verify_oracle}, except that the oracle e-values are replaced by the practical plug-in e-values. The time-scheduled mixture takes the form $q_t=(1-\rho_t-\beta_t)q_t^{\mathrm{grow}}+\rho_t q_t^{\mathrm{refine}}+\beta_t\mathbf 1/N$, where $q_t^{\mathrm{grow}}$ is the growth-oriented rule in \eqref{eq:oracle_maxmin_sampling} but it is computed using the plug-in e-values, and $q_t^{\mathrm{refine}}$ is the refinement rule of \cite{wuCandes2026efficientevaluationllmperformance} for improving the prediction $\hat p_t$ ($q_t^{\mathrm{refine}}$ is equivalent to $h^{(t)}_a$ in their Equation (3)). At a high-level, we choose $\rho_t$ to decrease over time and $\beta_t$ to increase over time. The intuition is that early queries should focus more on refining the prediction, while later queries should place more weight on growth-oriented and uniform sampling. For the data-dependent mixture rule, instead of fixing the schedules of $\rho_t$ and $\beta_t$, we update them adaptively so that they are proportional to two data-dependent scores, $s_\rho(t)$ and $s_\beta(t)$, respectively. The score $s_\rho(t)$ measures the uncertainty of the current prediction and controls the weight on prediction refinement. The score $s_\beta(t)$ measures how much the question-level predictions have stabilized around similar values and controls the weight on the uniform sampling part. The intuition is that, once the predictions have been sufficiently refined and indicate that most questions have similarly high correctness probabilities, uniform sampling can already serve as an effective querying rule, as observed in the left panel of Figure~\ref{fig:oracle_avg_stopping_time}. Otherwise, we continue to prioritize queries that reduce prediction uncertainty. The remaining probability mass is allocated to $q_t^{\mathrm{grow}}$. 

Then, to test our proposed methods under different capabilities of new LLM models, we present the empirical comparison of these approaches on different capability regimes in Figure~\ref{fig:ripr_only_avg_stop_over_time} and Figure~\ref{fig:bet_only_avg_stop_over_time}. The main messages from this empirical result are as follows. First, there is \emph{no universally best querying rule} across all regimes; the most effective policy depends on the structure of the testing dataset. Second, the \emph{similarity between the new model and the historical models} plays a key role in sampling efficiency. In regimes such as \texttt{mean\_shift\_mismatch\_high\_accuracy}, \texttt{normal\_regime}, and \texttt{out\_of\_family\_model\_mismatch}, the new model differs more substantially from the historical model family, making prediction refinement harder and slowing the CS shrinkage. Again, the testing-by-betting-based approach appears more robust to this mismatch than the RIPr-based approach. Third, the \emph{testing-by-betting-based CS with uniform sampling} achieves competitive performance across all regimes, suggesting that a simple querying rule combined with a computationally simpler CS construction may already be sufficient in practice. The details of this experiment are deferred to Appendix~\ref{append:experiment_details_practical}. We end this section with the following remark.

\begin{figure}[t]
    \centering
    \includegraphics[width=0.95\textwidth]{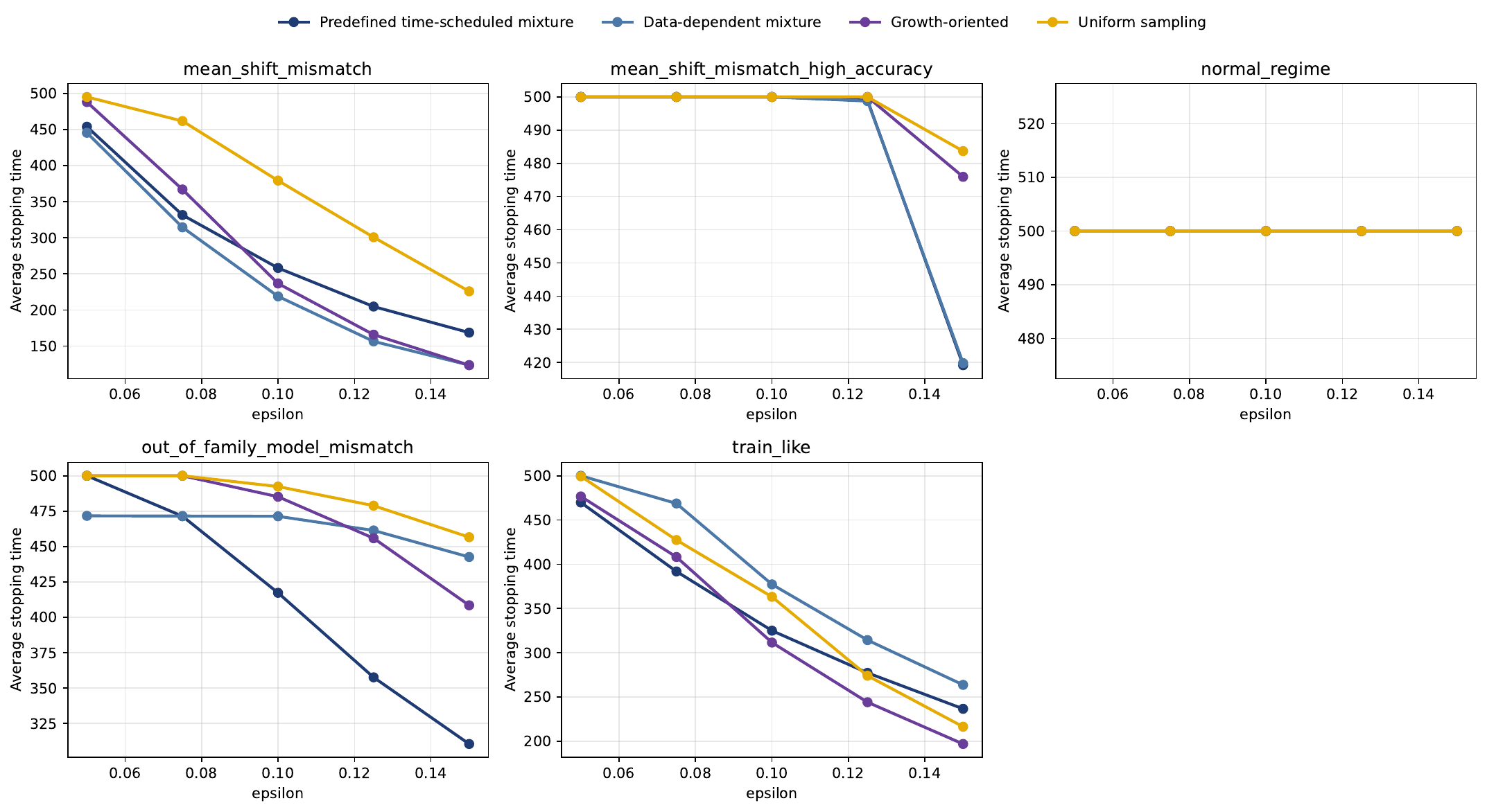}
    \caption{Average stopping time versus $\epsilon$ for RIPr-based methods across five synthetic regimes: \texttt{train\_like}, which is close to the training distribution; \texttt{mean\_shift\_mismatch} and \texttt{mean\_shift\_mismatch\_high\_accuracy}, which have shifted overall accuracy levels; \texttt{out\_of\_family\_model\_mismatch}, which exhibits structural mismatch from the training family; and \texttt{normal\_regime}, which is a handcrafted heterogeneous setting with a substantially different average accuracy from the historical models. We compare four sampling rules.} 
\label{fig:ripr_only_avg_stop_over_time}
\end{figure}
\begin{figure}[t]
    \centering
    \includegraphics[width=0.95\textwidth]{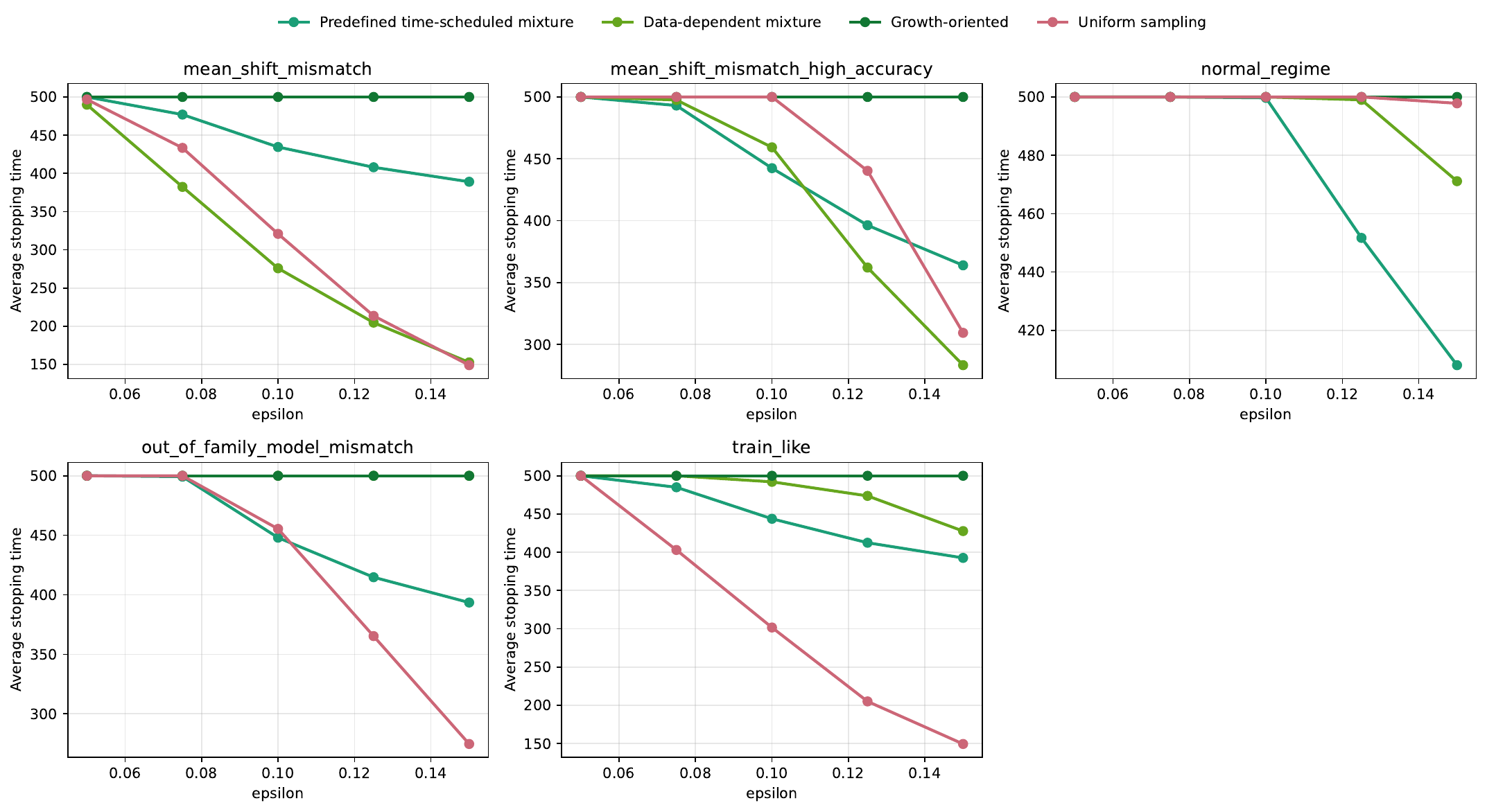}
    \caption{Average stopping time versus $\epsilon$ for testing-by-betting-based methods across five synthetic regimes described in Figure~\ref{fig:ripr_only_avg_stop_over_time}.} 
    \label{fig:bet_only_avg_stop_over_time}
\end{figure}

\begin{remark}[Time-uniform bound of \cite{wuCandes2026efficientevaluationllmperformance}]
By slightly modifying the arguments in Section~\ref{sec:practical}, one can also derive a time-uniform uncertainty bound for the estimator of \cite{wuCandes2026efficientevaluationllmperformance}. However, the resulting bound is relatively loose in our setting and typically requires more queries to reach the same target accuracy $\epsilon$ than our proposed sequential methods in the experiments. For this reason, we do not include it as a main baseline. %
\end{remark}

\section{Conclusion}\label{sec:conclusion}
We studied active querying for LLM evaluation with time-uniform uncertainty control. We constructed confidence sequences based on RIPr and testing-by-betting, and analyzed the shrinkage behavior of the resulting CSs. This analysis identified that prediction mismatch and the spikiness of the querying distribution affect the sampling efficiency, and motivated the querying rules considered in our experiments. Our empirical results suggest that there is no uniformly best querying rule across all testing datasets. Interestingly, we also observe that uniform sampling is a competitive baseline for both methods. Future work includes developing more effective querying rules, possibly using reinforcement learning. Moreover, it would be worthwhile to replace the Bayesian factor model with other prediction methods, such as the collaborative evaluation approach of \cite{fisch2026collabevalstatisticallyefficientcollaborative}.  %

\newpage
\bibliographystyle{IEEEtran}
\bibliography{ref}

\newpage
\appendix

\section{Experiment details and supplemental results}\label{append:experiment_details}

\subsection{A heuristic algorithm of the CS}
We present a \emph{heuristic} version of the general algorithm described in Section~\ref{sec:prelim}. This algorithm serves as a unified template: it can be instantiated with either the RIPr-based or the testing-by-betting-based approach by plugging in the corresponding conditional e-values and querying rules.
\begin{center} 
\begin{minipage}{0.95\linewidth} 
\textbf{Algorithm 1: A general construction of the confidence sequence} 
\begin{algorithmic}[1]
\State \textbf{Input:} Question bank $[N]$, historical data $H$, confidence level $\alpha$, stopping criterion $\mathcal T$, grid $\mathcal M_K\subset[0,1]$ with $K$ grid points.
\State \textbf{Output:} A stopped confidence sequence or the full sequence $\{C_t\}_{t=0}^{N}$.

\State Initialize $W_0(m)\gets 1$ for all $m\in\mathcal M_K$, $C_0\gets \mathcal M_K$, and the reference information $r_1$ from $H$.

\For{$t=1,2,\ldots,N$}
    \State Update the predictable sampling distribution
    \[
    q_t\leftarrow
    \texttt{UpdateQueryPmf}
    \left(H,r_t,C_{t-1},\{I_s,Z_{I_s}\}_{s=1}^{t-1}\right).
    \]
    \State Sample $I_t\sim q_t$ and observe $Z_{I_t}\in\{0,1\}$.
    \For{each $m\in\mathcal M_K$}
        \State Construct the one-step conditional e-value
        \[
        E_t(m)\leftarrow
        \texttt{UpdateEvalue}
        \left(m,H,r_t,q_t,\{I_s,Z_{I_s}\}_{s=1}^{t}\right).
        \]
        \State Update the e-process $W_t(m)\gets W_{t-1}(m)E_t(m)$.
    \EndFor

    \State Update the grid-based heuristic set $
    \left\{
    m\in\mathcal M_K:
    W_t(m)<\frac{1}{\alpha}
    \right\}$ and let $C_t=[L_t, U_t]$, where $U_t:=\sup\left\{
    m\in\mathcal M_K:
    W_t(m)<\frac{1}{\alpha}
    \right\}$ and $L_t:=\inf\left\{
    m\in\mathcal M_K:
    W_t(m)<\frac{1}{\alpha}
    \right\}$.
    \If{$\mathcal T(t,C_t)=1$}
        \State \textbf{return} $t$ and $C_t$
    \EndIf

    \State Update the side information $r_{t+1}\leftarrow
    \texttt{UpdateReferenceInfo}
    \left(H,\{I_s,Z_{I_s}\}_{s=1}^{t}\right)$.
\EndFor
\State \textbf{return} $\{C_t\}_{t=0}^{N}$.
\end{algorithmic} 
\end{minipage} 
\end{center}
The reference information $\{r_t\}_{t\ge1}$ is specified differently in the oracle and practical settings. In the oracle setting, we set $r_t=p^\star$ for all $t\ge1$. In the practical setting, we set $r_t=\hat p_{t-1}$, where $\hat p_{t-1}$ is $\mathcal{F}_{t-1}$-measurable and updated over time using the Bayesian factor model and the Laplace update proposed in \cite{wuCandes2026efficientevaluationllmperformance}. The construction of the CS is a heuristic method. It is easy to implement this method, but we emphasize that it is \emph{not} in general a valid level-$\alpha$ CS. Under this unified formulation, the RIPr-based and testing-by-betting-based approaches differ only in the two subroutines \texttt{UpdateQueryPmf} and \texttt{UpdateEvalue}, which specify how the querying distribution and the conditional e-value are updated.

\subsection{Oracle RIPr-based CS vs. Oracle Betting-based CS}\label{append:experiment_details_oracle}
We provide the implementation details of the simulation setting, CSs, and querying rules for the RIPr-based and testing-by-betting-based approaches used in Section~\ref{sec:exp_verify_oracle}.
\paragraph{Simulation setting.}
In this experiment, we first specify the new model's latent question-level correctness probabilities $p^\star=(p_1^\star,\ldots,p_N^\star)$. This vector is used to generate realized benchmark outcomes and, in the oracle setting, is also assumed to be available as the reference vector for constructing the oracle e-values and querying rules. To mimic a realistic benchmark evaluation, for each $b\in[B]$, we generate an independent binary correctness vector
\[
A^{(b)}=(A_1^{(b)},\ldots,A_N^{(b)}),
\qquad
A_i^{(b)}\sim \mathrm{Ber}(p_i^\star)
\]
independently across questions. Conditional on this realized benchmark vector, the target parameter is the realized benchmark accuracy
\[
\theta_A^{(b)}
=
\frac{1}{N}\sum_{i=1}^N A_i^{(b)}.
\]
This fixed-benchmark interpretation does not invalidate the CS guarantees in Proposition~\ref{prop:ripr_validity} and Proposition~\ref{prop:test_by_bet_validity}. Indeed, conditional on $A^{(b)}$, the observation distribution under a predictable querying rule $q_t$ is the degenerate Bernoulli model with parameter vector $A^{(b)}\in\{0,1\}^N$. Since our conditional e-value property holds uniformly over all $p\in\mathcal P_m$, it also holds for the boundary vector $p=A^{(b)}$. Therefore, when $m=\theta_A^{(b)}$, the wealth process remains a test supermartingale under the randomness of the adaptive querying procedure, and the resulting CS is valid for the realized target $\theta_A^{(b)}$.

We compare the RIPr-based CS and the testing-by-betting-based CS in terms of the average number of questions required for the CS width to fall below a prescribed threshold $\epsilon$. The first time that the width falls below $\epsilon$ is called the \emph{stopping time}. Equivalently, the stopping criterion satisfies $\mathcal T(t,C_t)=1$ if $|C_t|\le \epsilon$, and the corresponding $t$ is the stopping time. We consider the set of $\epsilon$ values described in Figure~\ref{fig:oracle_avg_stopping_time}. For each $\epsilon$ and each realized correctness vector $A^{(b)}$, we run the sequential querying procedure $R$ times with independent querying randomness. We set $B=5$ and $R=10$, and report the average stopping time over all $B\times R$ runs for each value of $\epsilon$. Moreover, we uniformly discretize the domain $[0,1]$ into a finite grid $\mathcal M_K$ with $K=200$ points and set the anytime-valid error level to $\alpha=0.05$.

We consider two regimes of capabilities (the $p^\star$) and $N=500$. In the \texttt{almost\_all\_correct} regime, $0.9N$ questions have correctness probabilities drawn in $[0.985,0.999]$, while the remaining $0.1N$ questions have correctness probabilities drawn in  $[0.9,0.97]$. In the \texttt{normal\_level} regime, $0.05N$ questions have correctness probabilities drawn in $[0.15,0.35]$, $0.2N$ questions have correctness probabilities drawn in  $[0.7,0.8]$, and the remaining $0.75N$ questions have correctness probabilities drawn in  $[0.985,0.995]$.

\paragraph{RIPr-based approach.}
We introduce the pseudo-code of the functions \texttt{UpdateQueryPmf} and \texttt{UpdateEvalue} for the RIPr-based approach.
\begin{center}
\begin{minipage}{0.95\linewidth}
\textbf{Subroutine 1: \texttt{RIPr-UpdateQueryPmf}}
\begin{algorithmic}[1]
\State \textbf{Input:} Reference vector $r_t=p^\star$, current CS $C_{t-1}$, smoothing weight $\zeta$, number of EG iterations $R_{\mathrm{EG}}=25$.
\State \textbf{Output:} Querying distribution $q_t$.
\If{$\zeta=1$}
\State \textbf{return} $q_t\gets \mathbf 1/N$.
\EndIf
\State Set $L_{t-1}\gets \inf C_{t-1}$ and $U_{t-1}\gets \sup C_{t-1}$.
\State Initialize $q^{(1)}\gets \mathbf 1/N$ and $q^{\mathrm{best}}\gets q^{(1)}$.
\State Define the growth objective
\[
J_t(q)
:=
\min_{m\in\{L_{t-1},U_{t-1}\}}
G_t\bigl(q,E_t^{\mathrm{RIPr}}(m,r_t, q),r_t\bigr).
\]

\For{$i=1,2,\ldots,R_{\mathrm{EG}}$}
\State Compute a subgradient $g^{(i)}$ of $J_t(q)$ at $q=q^{(i)}$.
\State Set the EG learning rate $\eta_i\gets 1/(2\sqrt{i})$.
\State Apply the exponentiated-gradient update:
\[
\widetilde q^{(i+1)}_j
\gets
q^{(i)}_j\exp(\eta_i g^{(i)}j),
\qquad j\in[N].
\]
\State Normalize:
\[
q^{(i+1)}\gets
\frac{\widetilde q^{(i+1)}}{\sum_{j=1}^N \widetilde q^{(i+1)}_j}.
\]
\If{$J_t(q^{(i+1)})>J_t(q^{\mathrm{best}})$}
\State Set $q^{\mathrm{best}}\gets q^{(i+1)}$.
\EndIf
\EndFor

\State \textbf{return} the smoothed querying distribution $q_t\gets (1-\zeta)q^{\mathrm{best}}+\zeta\frac{\mathbf 1}{N}$.
\end{algorithmic}
\end{minipage}
\end{center}

We utilize the exponentiated-gradient (EG) updates for approximating the maximizer of $J_t(q)$. Moreover, we return the smoothed version to satisfy the requirement in Proposition~\ref{prop:ripr_validity} for a valid conditional e-value.

\begin{center}
\begin{minipage}{0.95\linewidth}
\textbf{Subroutine 2: \texttt{RIPr-UpdateEvalue}} 
\begin{algorithmic}[1]\label{alg:oracl-ripr-e-value}
\State \textbf{Input:} Candidate $m\in\mathcal M_K$, reference vector $r_t=p^\star$, querying distribution $q_t$, queried question $I_t$, observed outcome $Z_{I_t}$.
\State \textbf{Output:} Conditional e-value $E_t^{\mathrm{RIPr}}(m,q_t)$.

\State Compute the RIPr point $p_{q_t}^\dagger(m)$ associated with the null candidate $m$, querying distribution $q_t$, and reference vector $r_t$.

\State To compute $p_{q_t}^\dagger(m; r_t)$, solve for the Lagrange multiplier $\lambda^\star(m)$ in Proposition~\ref{prop:solve_ripr_point}. Initialize $\lambda_{\mathrm{lo}}=-1$ and $\lambda_{\mathrm{hi}}=1$. We do the bracket extension followed by the bisection as follows.

\For{$j=1,2,\ldots,24$}
\If{$H(\lambda_{\mathrm{lo}})<m$}
\State Set $\lambda_{\mathrm{lo}}\gets 2\lambda_{\mathrm{lo}}$.
\EndIf
\If{$H(\lambda_{\mathrm{hi}})>m$}
\State Set $\lambda_{\mathrm{hi}}\gets 2\lambda_{\mathrm{hi}}$.
\EndIf
\If{$H(\lambda_{\mathrm{lo}})\ge m\ge H(\lambda_{\mathrm{hi}})$}
\State Break.
\EndIf
\EndFor

\For{$j=1,2,\ldots,28$}
\State Set $\lambda_{\mathrm{mid}}\gets(\lambda_{\mathrm{lo}}+\lambda_{\mathrm{hi}})/2$.
\If{$H(\lambda_{\mathrm{mid}})>m$}
\State Set $\lambda_{\mathrm{lo}}\gets \lambda_{\mathrm{mid}}$.
\Else
\State Set $\lambda_{\mathrm{hi}}\gets \lambda_{\mathrm{mid}}$.
\EndIf
\EndFor

\State Set $\lambda^\star(m)\gets(\lambda_{\mathrm{lo}}+\lambda_{\mathrm{hi}})/2$ and construct $p_{q_t}^\dagger(m;r_t)$ according to Proposition~\ref{prop:solve_ripr_point}. \State \textbf{return} $E_t^{\mathrm{RIPr}}(m,q_t)
=
\frac{
(r_t(I_t))^{Z_{I_t}}(1-r_t(I_t))^{1-Z_{I_t}}
}{
(p_{q_t}^\dagger(m;r_t)_{I_t})^{Z_{I_t}}(1-p_{q_t}^\dagger(m;r_t)_{I_t})^{1-Z_{I_t}}
}$.
\end{algorithmic}
\end{minipage}
\end{center}

\paragraph{Testing-by-betting-based approach.}
We introduce the pseudo-code of the functions \texttt{UpdateQueryPmf} and \texttt{UpdateEvalue} for the testing-by-betting-based approach.

\begin{center}
\begin{minipage}{0.95\linewidth}
\textbf{Subroutine 1: \texttt{Betting-UpdateQueryPmf}}
\begin{algorithmic}[1]
\State \textbf{Input:} Reference vector $r_t=p^\star$, current CS $C_{t-1}$, smoothing weight $\zeta$, number of EG iterations $R_{\mathrm{EG}}$, number of bisection iterations $B_{\mathrm{bet}}$. We set $B_{\mathrm{bet}}=40$ and $R_{\mathrm{EG}}=25$.
\State \textbf{Output:} Querying distribution $q_t$.

\If{$\zeta=1$}
\State \textbf{return} $q_t\gets \mathbf 1/N$.
\EndIf

\State Set $L_{t-1}\gets \inf C_{t-1}$ and $U_{t-1}\gets \sup C_{t-1}$.
\State Use the known lower bound $h=\zeta/N$ guaranteed by the smoothing step.
\State Define the smoothing operator
\[
\operatorname{Mix}_{\zeta}(q)
:=
(1-\zeta)q+\zeta\frac{\mathbf 1}{N}.
\]
\State Initialize
\[
q^{(1)}
\gets
\mathbf{1}/N\qquad
q^{\mathrm{best}}\gets q^{(1)}.
\]

\State Define the growth objective
\[
J_t(q)
:=
\min_{m\in\{L_{t-1},U_{t-1}\}}
G_t\bigl(q,E_t^{\mathrm{bet}}(m,\lambda_t^\star(m,q),r_t,q),r_t)
\]
where
\[
\lambda_t^\star(m,q)
\in
\arg\max_{\lambda\in\Lambda_h(m)}
G_t\bigl(q,E_t^{\mathrm{bet}}(m,\lambda,r_t,q),r_t\bigr).
\]

\For{$i=1,2,\ldots,R_{\mathrm{EG}}$}
\For{each endpoint $m\in\{L_{t-1},U_{t-1}\}$}
\State Compute
\[
\lambda_t^\star(m,q^{(i)})
\in
\arg\max_{\lambda\in\Lambda_h(m)}
G_t\bigl(q^{(i)},E_t^{\mathrm{bet}}(m,\lambda,r_t,q^{(i)}),r_t\bigr)
\]
by bisection over $\Lambda_h(m)$ for $B_{\mathrm{bet}}$ iterations, with projection onto the endpoints of $\Lambda_{h}(m)$ when the maximizer is attained at the boundary.
\EndFor

\State Evaluate $J_t(q^{(i)})$ and compute a subgradient $g^{(i)}$ at $q=q^{(i)}$.

\If{$J_t(q^{(i)})>J_t(q^{\mathrm{best}})$}
    \State Set $q^{\mathrm{best}}\gets q^{(i)}$.
\EndIf

\State Set the EG learning rate $\eta_i\gets 0.5/\sqrt{i}$.

\State Apply the exponentiated-gradient update:
\[
\widetilde q^{(i+1)}_j
\gets
q^{(i)}_j\exp\bigl(\eta_i g^{(i)}_j\bigr),
\qquad j\in[N].
\]

\State Normalize:
\[
\hat q^{(i+1)}
\gets
\frac{\widetilde q^{(i+1)}}{\sum_{j=1}^N\widetilde q^{(i+1)}_j}.
\]

\State Smooth:
\[
q^{(i+1)}
\gets
\operatorname{Mix}_{\zeta}\bigl(\hat q^{(i+1)}\bigr).
\]
\EndFor

\If{$J_t(q^{(R_{\mathrm{EG}}+1)})>J_t(q^{\mathrm{best}})$}
\State Set $q^{\mathrm{best}}\gets q^{(R_{\mathrm{EG}}+1)}$.
\EndIf

\State \textbf{return} $q_t\gets q^{\mathrm{best}}$.
\end{algorithmic}
\end{minipage}
\end{center}
We also use EG to approximate the maximizer of $J_t(q)$ for the betting-based method. The smoothing step at the end of each EG iteration helps prevent the optimized bet $\lambda_t^\star(\cdot)$ from becoming zero after the EG update. For computing the $\lambda_t^{\star}(\cdot)$, since the objective is one-dimensional and concave in $\lambda$, we compute $\lambda_t^\star(\cdot)$ by applying bisection to the first-order condition for $B_{\mathrm{bet}}$ iterations, with projection onto the endpoints of $\Lambda_{h}(m)$ when the maximizer is attained at the boundary.

\begin{center}
\begin{minipage}{0.95\linewidth}
\textbf{Subroutine 2: \texttt{Betting-UpdateEvalue}}
\begin{algorithmic}[1]\label{alg:oracle-betting-e-value}
\State \textbf{Input:} Candidate $m\in\mathcal M_K$, reference vector $r_t=p^\star$, querying distribution $q_t$, queried question $I_t$, observed outcome $Z_{I_t}$, number of bisection iterations $B_{\mathrm{bet}}=40$.
\State \textbf{Output:} Conditional e-value $E_t^{\mathrm{bet}}(m,q_t;r_t)$.

\State Use the known lower bound $h=\zeta/N$ guaranteed by the smoothing step.
\State Compute the optimal betting parameter
\[
\lambda_t^\star(m,q_t)
\in
\arg\max_{\lambda\in\Lambda_{h}(m)}
G_t\bigl(q_t,E_t^{\mathrm{bet}}(m,\lambda, r_t, q_t),m\bigr).
\]
following the bisection method stated in the \texttt{Betting-UpdateQueryPmf}.

\State \textbf{return} $E_t^{\mathrm{bet}}(m,\lambda_t^\star(m,q_t),r_t,q_t) = 1+\lambda_t^\star(m,q_t)\bigl(\Phi_t(r_t,q_t)-m\bigr)$, where $\Phi_t(r_t,q_t):=\frac{1}{N}\sum_{i=1}^N r_t(i)
+
\frac{Z_{I_t}-r_t(I_t)}{Nq_t(I_t)}$.
\end{algorithmic}
\end{minipage}
\end{center}

Finally, Figure~\ref{fig:oracle_coverage_over_time} reports the empirical coverage results, which are consistent with the nominal level $1-\alpha=0.95$, up to Monte Carlo variability due to the limited number of repetitions. Figure~\ref{fig:oracle_width_over_time} shows the corresponding CS width over time.

\FloatBarrier
\begin{figure}[t]
    \centering
    \includegraphics[width=0.95\textwidth]{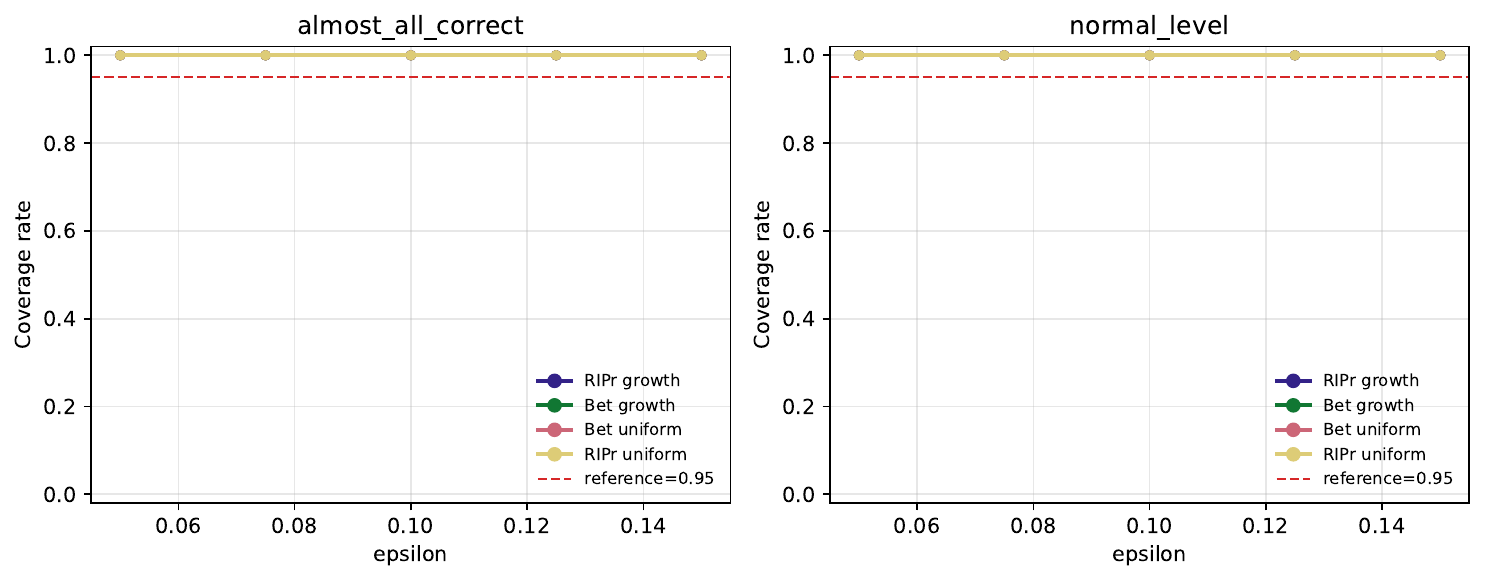}
    \caption{Oracle setting: anytime coverage rate for each target accuracy $\epsilon$. For each value of $\epsilon$, we check whether $\theta_A^{(b)}\in C_t$ for each $b\in[B]$ holds over all query rounds until reaching the stopping criterion for each $\epsilon$ value ($|C_t|\leq\epsilon$). The reported value is the rate of $\mathbf{1}\{\forall\,t\leq\tau_\epsilon,\,\theta_A^{(b)}\in [L_t, U_t]\}=1$, where $\tau_\epsilon:=\inf\{t\ge1:U_t-L_t\leq\epsilon\}$, over $R\times B$ repetitions. The dashed line indicates the nominal coverage level $1-\alpha=0.95$. Empirical coverage above this line is consistent with the level-$\alpha$ guarantee. The curves overlap in this figure.}
    \label{fig:oracle_coverage_over_time}
\end{figure}
\begin{figure}[t]
    \centering
    \includegraphics[width=0.8\textwidth]{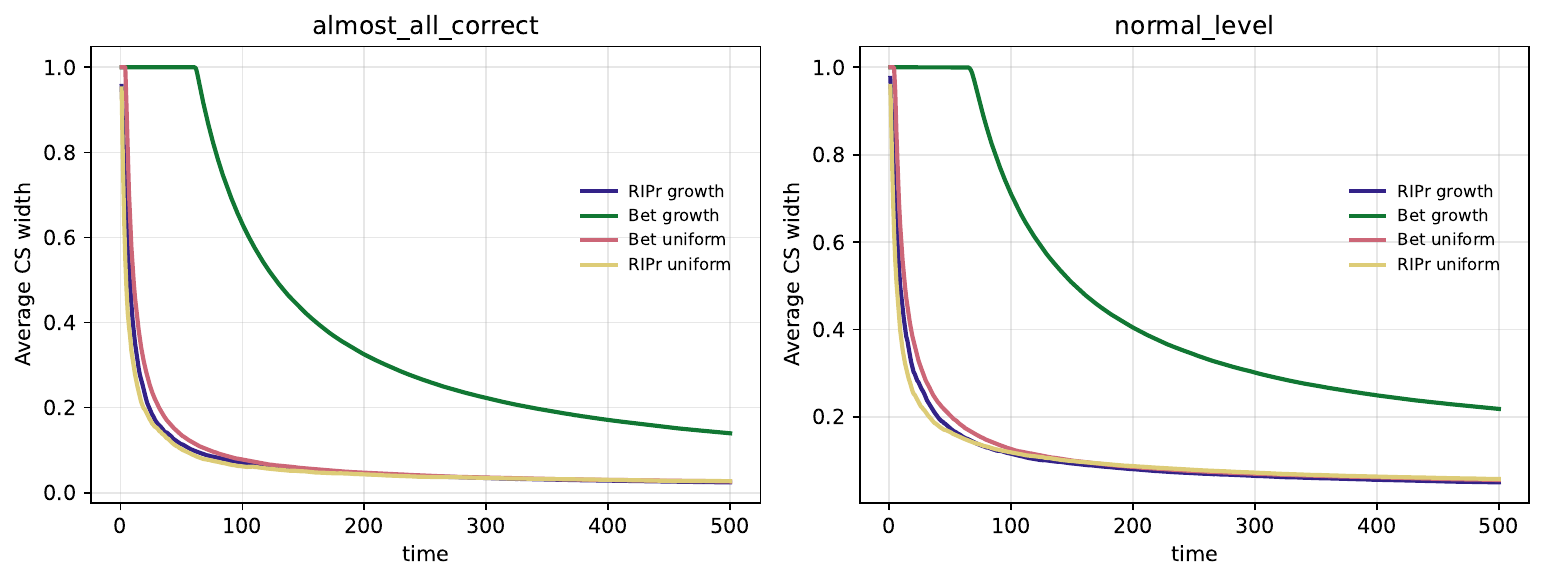}
    \caption{Average CS width over time. It is obtained in the same experiments conducted in Figure~\ref{fig:oracle_avg_stopping_time}.} 
    \label{fig:oracle_width_over_time}
\end{figure}

\subsection{Practical RIPr-based CS vs. Practical betting-based CS}\label{append:experiment_details_practical}
\paragraph{Simulation setting.}
This experiment follows the same setup as Appendix~\ref{append:experiment_details_oracle}, except that the ground-truth answer sets are generated under five different capability regimes, the grid size is set to $K=400$, and we use $R=2$ querying runs for each of $B=8$ independent ground-truth realizations. We next describe how these five capability regimes are constructed. For the training data, we apply the MMLU-PRO used in \cite{wuCandes2026efficientevaluationllmperformance}.

We start from the factor model fitted on the historical data. Following the factorization method in \cite{wuCandes2026efficientevaluationllmperformance}, the fitted factor model provides a matrix of historical model embeddings and a matrix of question embeddings, which together define a low-rank latent representation of question-level correctness probabilities. At a high level, the \texttt{train\_like} regime represents a new model whose question-level behavior is broadly similar to the historical models. The \texttt{mean\_shift\_mismatch} regime represents a model whose overall capability level is shifted downward while its question-level pattern remains broadly training-like. The \texttt{mean\_shift\_mismatch\_high\_accuracy} regime represents a model with substantially higher overall accuracy but still a largely training-like question pattern. The \texttt{out\_of\_family\_model\_mismatch} regime represents a model whose overall mean accuracy is not necessarily very different, but whose pattern of easy and hard questions is intentionally less similar to the historical model family. Finally, the \texttt{normal\_regime} is a hand-crafted heterogeneous benchmark consisting of a large easy subset, a moderate intermediate subset, and a small difficult subset.

The details of generating these regimes are as follows: Let $\{u_1,\dots,u_M\}\subset\mathbb R^d$ denote the historical model embeddings and let $\{v_1,\dots,v_N\}\subset\mathbb R^d$ denote the question embeddings learned from the historical data by \cite[Appendix A.1]{wuCandes2026efficientevaluationllmperformance}, after restricting attention to the $N=500$ questions used in the experiment. Let
\[
\mu_U := \frac{1}{M}\sum_{m=1}^M u_m,
\qquad
\Sigma_U := \mathrm{Cov}(u_1,\dots,u_M)
\]
be the empirical mean and covariance of the historical model embeddings. For each factor-model-based regime $r$, we generate a candidate pool of synthetic question-level probability vectors as follows. First, for each candidate index $c\in[M_{\mathrm{pool}}]$, we sample
\[
\widetilde u_r^{(c)} \sim \mathcal N(\mu_U,\Sigma_U),
\qquad
u_r^{(c)} := \mu_U + a_r\bigl(\widetilde u_r^{(c)}-\mu_U\bigr),
\]
where $a_r>0$ is a regime-specific latent scaling parameter. We then form question-level logits
\[
\ell_{r,j}^{(c)}
=
b_r\langle u_r^{(c)}, v_j\rangle
+\gamma_r
+\xi_{r,j}^{(c)},
\qquad j\in[N],
\]
where $b_r>0$ is a regime-specific logit scaling parameter, $\gamma_r$ is a global shift chosen to control the overall mean accuracy level, and $\xi_{r,j}^{(c)}$ is a regime-specific perturbation term. The resulting question-level correctness probabilities are
\[
p_{r,j}^{(c)} = \sigma\!\bigl(\ell_{r,j}^{(c)}\bigr),
\qquad
\sigma(x)=\frac{1}{1+e^{-x}}.
\]
Repeating this procedure yields a candidate pool $\{p_r^{(c)}\}_{c=1}^{M_{\mathrm{pool}}}\subset(0,1)^N$. We then choose one representative candidate from this pool as the regime-specific ground-truth vector $p_r^\star$, using a regime-dependent score that favors either similarity to or mismatch from the historical model family.

To describe these regime-dependent selection rules, let $\mathcal R$ denote a reference subset of historical probability rows, let
\[
\bar p(p):=\frac{1}{N}\sum_{j=1}^N p_j
\]
be the average correctness probability of a candidate vector $p\in[0,1]^N$, and let
\[
d_{\mathrm{RMSE}}(p,\mathcal R)
:=
\min_{q\in\mathcal R}
\Bigl(\frac{1}{N}\|p-q\|_2^2\Bigr)^{1/2}
\]
denote the nearest root-mean-square deviation from the historical reference set. We also define
\[
\rho_{\max}(p,\mathcal R)
:=
\max_{q\in\mathcal R}\mathrm{Corr}(p,q),
\qquad
\mathrm{pen}_{\mathrm{corr}}(p,\mathcal R)
:=
1-\rho_{\max}(p,\mathcal R),
\]
so that a larger correlation penalty indicates a less training-like question pattern. The \texttt{train\_like} regime selects a candidate with mean accuracy and question-level structure close to the historical family; \texttt{mean\_shift\_mismatch} and \texttt{mean\_shift\_mismatch\_high\_accuracy} select candidates whose question-level structure remains broadly training-like but whose overall mean accuracy is shifted downward or upward, respectively; and \texttt{out\_of\_family\_model\_mismatch} selects a candidate whose question-level structure is intentionally less similar to the historical family.

More concretely, the \texttt{train\_like} regime uses $(a_r,b_r)=(1,1)$ together with only mild residual perturbations and chooses a candidate with small $|\bar p(p)-\bar p_{\mathrm{hist}}|$, small $d_{\mathrm{RMSE}}(p,\mathcal R)$, and small $\mathrm{pen}_{\mathrm{corr}}(p,\mathcal R)$, where $\bar p_{\mathrm{hist}}$ denotes the average correctness level of the historical reference family. The \texttt{mean\_shift\_mismatch} regime uses the same latent geometry but chooses $\gamma_r$ so that $\bar p(p)$ is close to a lower target mean while still favoring structural similarity to $\mathcal R$. The \texttt{mean\_shift\_mismatch\_high\_accuracy} regime is defined analogously, except that the target mean is set to a substantially larger value (in our implementation, close to $0.97$). The \texttt{out\_of\_family\_model\_mismatch} regime increases the latent scaling parameter $a_r$ and then selects a candidate with larger $d_{\mathrm{RMSE}}(p,\mathcal R)$ and larger $\mathrm{pen}_{\mathrm{corr}}(p,\mathcal R)$, thereby inducing a different pattern of easy and hard questions while keeping the overall accuracy level similar.

The \texttt{normal\_regime} is constructed separately rather than selected from the factor-model candidate pool. In this case, the ground-truth vector $p^\star$ is formed directly by assigning $p_j^\star \in \{0.9,\,0.7,\,0.1\}$
to the questions so that $80\%$ of the coordinates equal $0.9$, $15\%$ equal $0.7$, and the remaining $5\%$ equal $0.1$, followed by a random permutation of the question order. This yields a heterogeneous benchmark with a large easy subset, a moderate intermediate subset, and a small difficult subset. The full implementation and parameter setting are provided in the accompanying GitHub code, \texttt{select\_regime\_models.py}.

A brief summary of each regime used in our experiments is in Table~\ref{tab:mmlu-pro-regime-setting}. Table~\ref{tab:mmlu-pro-regime-row-means} reports the average question-level correctness probability for each synthetic regime. The regimes \texttt{train\_like} and \texttt{mean\_shift\_mismatch} are structurally similar to the historical data, whereas the other three regimes exhibit larger discrepancies from the historical model family. Therefore, the prediction mismatch is larger in these three regimes, leading to worse performance on sampling efficiency.

\begin{table}[t]
\centering
\caption{Parameter settings for the synthetic capability regimes.}
\label{tab:mmlu-pro-regime-setting}
\small
\begin{tabular}{lcccc}
\toprule
Regime & $a_r$ & $b_r$ & Target mean & Selection preference \\
\midrule
\texttt{train\_like} 
& $1.0$ & $1.0$ & $\bar p_{\mathrm{hist}}$ 
& small RMSE, high corr \\

\texttt{mean\_shift\_mismatch} 
& $1.0$ & $1.0$ & lower target 
& small RMSE, high corr \\

\texttt{mean\_shift\_mismatch\_high\_accuracy} 
& $1.0$ & $1.0$ & $0.97$ 
& small RMSE, high corr \\

\texttt{out\_of\_family\_model\_mismatch} 
& $1.8$ & $1.0$ & training-like 
& large RMSE, low corr \\

\texttt{normal\_regime} 
& -- & -- & handcrafted 
& -- \\
\bottomrule
\end{tabular}
\end{table}

\begin{table}[t]
\centering
\caption{Average question-level correctness probability $\frac{1}{N}\sum_{i=1}^N p^\star_i$ for each synthetic regime constructed from \texttt{mmlu-pro} when $N=500$. The historical training-table average is $0.2598$.}
\label{tab:mmlu-pro-regime-row-means}
\begin{tabular}{lc}
\toprule
Regime & Row mean \\
\midrule
\texttt{train\_like} & 0.2540 \\
\texttt{mean\_shift\_mismatch} & 0.1619 \\
\texttt{mean\_shift\_mismatch\_high\_accuracy} & 0.9443 \\
\texttt{out\_of\_family\_model\_mismatch} & 0.4192 \\
\texttt{normal\_regime} & 0.8300 \\
\bottomrule
\end{tabular}
\end{table}

\paragraph{Sampling rules.}

The pseudo-code of the sampling rules for the predefined time-scheduled mixture and the data-dependent mixture is summarized below.
\begin{center}
\begin{minipage}{0.95\linewidth}
\textbf{Subroutine: \texttt{UpdateQueryPmf}}
\begin{algorithmic}[1]
\State \textbf{Input:} Current time $t$, previous prediction vector $\hat p_{t-1}\in[0,1]^N$, refinement scores $\{d_{t,j}\}_{j=1}^N$, mixture mode $\mathtt{mode}\in\{\mathtt{scheduled},\mathtt{data}\}$, mixture hyperparameters $\rho_{\mathrm{init}}$, scheduled parameter $a$, $l$, and $m_0$. $\beta_{\mathrm{target}}$.
\State \textbf{Output:} Querying distribution $q_t$.

\State Construct the refinement-oriented distribution by
\[
q_t^{\mathrm{refine}}(j)
\gets\mathrm{Norm}(\{d_{t,j}\}_{j=1}^{N})(j),\qquad j\in[N].
\]

\State Obtain $q_t^{\mathrm{grow}}$ by applying \texttt{RIPr-UpdateQueryPmf} or \texttt{Betting-UpdateQueryPmf} with $\zeta=0.05$ and reference vector $r_t=\hat p_{t-1}$. 

\If{$\mathtt{mode}=\mathtt{scheduled}$}
    \State Set $x_t\gets t/N$.
    \State Set
    \[
    \rho_t\gets \rho_{\mathrm{init}}(1-x_t)^a,
    \qquad
    \beta_t\gets \beta_{\mathrm{target}}\operatorname{NormLogistic}(x_t;l,m_0),
    \qquad
    w_t^{\mathrm{grow}}\gets 1-\rho_t-\beta_t.
    \]
\ElsIf{$\mathtt{mode}=\mathtt{data}$}
    \State Compute the refinement score
    \[
    s_\rho(t)\gets \operatorname{RefineSignal}\bigl(\{d_{t,j}\}_{j=1}^N\bigr).
    \]
    \State Compute the uniform-preference score
    \[
    s_\beta(t)\gets \operatorname{UniformPreference}(\hat p_{t-1}).
    \]
    \State Set
    \[
    \rho_t
    \gets
    \rho_{\mathrm{init}}s_\rho(t)(1-s_\beta(t)),
    \qquad
    \beta_t
    \gets
    \min\{\beta_{\mathrm{target}},1-\rho_t\}s_\beta(t),
    \qquad
    w_t^{\mathrm{grow}}
    \gets
    1-\rho_t-\beta_t.
    \]
\EndIf

\State Form the final querying distribution
\[
q_t
\gets
w_t^{\mathrm{grow}}q_t^{\mathrm{grow}}
+
\rho_t q_t^{\mathrm{refine}}
+
\beta_t\frac{\mathbf 1}{N}.
\]

\State Normalize for numerical stability:
\[
q_t
\gets
\frac{q_t}{\sum_{j=1}^N q_t(j)}.
\]

\State \textbf{return} $q_t$.
\end{algorithmic}
\end{minipage}
\end{center}

The refinement scores $d_{t,j}$ are equivalent to $d^{(t)}(j)$ in \cite[Appendix A.2.3]{wuCandes2026efficientevaluationllmperformance} for maximal posterior-variance-reducing. The norm logistic function is defined as $\operatorname{NormLogistic}(x_t;l,m_0)
=
\frac{\sigma\bigl(l(x_t-m_0)\bigr)-\sigma(-lm_0)}
{\sigma\bigl(l(1-m_0)\bigr)-\sigma(-lm_0)}$, where $\sigma(u)=1/(1+e^{-u})$. It is a smooth increasing function of the normalized time $x_t=t/N$, equal to $0$ at the beginning ($t=0$) and $1$ at the end of the horizon ($t=N$), with $l$ controlling the steepness and $m_0$ controlling the midpoint of the transition. Moreover, in the \texttt{scheduled} mode, the parameter $a$ additionally controls the decay rate of $\rho_t$, $\rho_{\mathrm{init}}$ sets its initial scale, and $\beta_{\mathrm{target}}$ controls the maximum target value of $\beta_t$. In the experiments for Figure~\ref{fig:ripr_only_avg_stop_over_time} and Figure~\ref{fig:bet_only_avg_stop_over_time}, we set $\rho_{\mathrm{init}} = 0.75, a = 1.5, \beta_{\mathrm{target}} = 0.7, l = 10$, and $m_0 = 0.7$. We set $\beta_{\mathrm{target}}<1$ because we expect the growth-oriented component to remain useful, although its benefit in this mixture design has not been theoretically analyzed and understood well. However, it is a natural surrogate for shrinking the CS width.

The function $\operatorname{RefineSignal}(\cdot)$ measures both the magnitude and the breadth of the refinement scores $\{d_{t,j}\}_{j=1}^N$, and returns the product of the two. The magnitude term captures how much larger the average top-$k$ refinement score is than the median score. A large magnitude indicates that some questions may be particularly useful for refining the prediction. However, this can also happen when only a small number of questions have very large scores. To avoid overemphasizing such highly concentrated cases, we introduce a breadth term based on the effective sample size of $q_t^{\mathrm{refine}}$: $\operatorname{ESS}_t
:=
\frac{1}{\sum_{j=1}^N (q_t^{\mathrm{refine}}(j))^2}$.
For example, if $q_t^{\mathrm{refine}}$ is uniform over $n$ questions, then $\operatorname{ESS}_t=n$. Thus, a larger ESS indicates that the refinement distribution is spread over more questions, while a smaller ESS indicates stronger concentration. Overall, $\operatorname{RefineSignal}(\cdot)$ becomes large only when the refinement scores have both a strong top-tail signal and sufficient breadth, suggesting that many questions are worth querying to improve the prediction.

The function $\operatorname{UniformPreference}(\cdot)$ measures how concentrated the entries of $\hat p_{t-1}$ are around a common level. We first compute the median of $\hat p_{t-1}$ and then calculate the fraction of coordinates that lie within a predefined tolerance $\delta$:
\[
\frac{1}{N}
\sum_{j=1}^{N}
\mathbf 1
\left\{
\left|
\hat p_{t-1}(j)
-\operatorname{median}\left(\{\hat p_{t-1}(\ell)\}_{\ell=1}^{N}\right)
\right|
\le \delta
\right\}.
\]
This fraction is then normalized to obtain a score representing how similar the predicted question-level correctness probabilities are. A larger value indicates that the predictions are more homogeneous, in which case uniform sampling becomes more favorable.

The implementations of $\operatorname{RefineSignal}(\cdot)$ and $\operatorname{UniformPreference}(\cdot)$ involve more predefined hyperparameters. We refer interested readers to the accompanying code, \texttt{mixture\_state\_dependent\_ripr\_cs\_final.py} and \texttt{mixture\_state\_dependent\_bet\_cs\_final.py}, for the full implementation details.

\paragraph{Practical e-value}
The practical conditional e-values for the RIPr-based and testing-by-betting-based approaches are updated using the subroutines \texttt{RIPr-UpdateEvalue} and \texttt{Betting-UpdateEvalue}, respectively, as described in Appendix~\ref{append:experiment_details_oracle}, with the reference vector set to $r_t=\hat p_{t-1}$, i.e., all the computations involving $p^\star$ are replaced by $\hat p_{t-1}$ at each round $t$.

Finally, we report the empirical coverage results in Figures~\ref{fig:ripr_only_coverage} and~\ref{fig:bet_only_coverage} to verify that the level-$\alpha$ guarantee is maintained. In both experiments, we set $\alpha=0.05$. We observe that most cases satisfy the target coverage level. A few cases fall slightly below the level, which we attribute to the Monte-Carlo error.

\begin{figure}[t]
    \centering
    \includegraphics[width=0.95\textwidth]{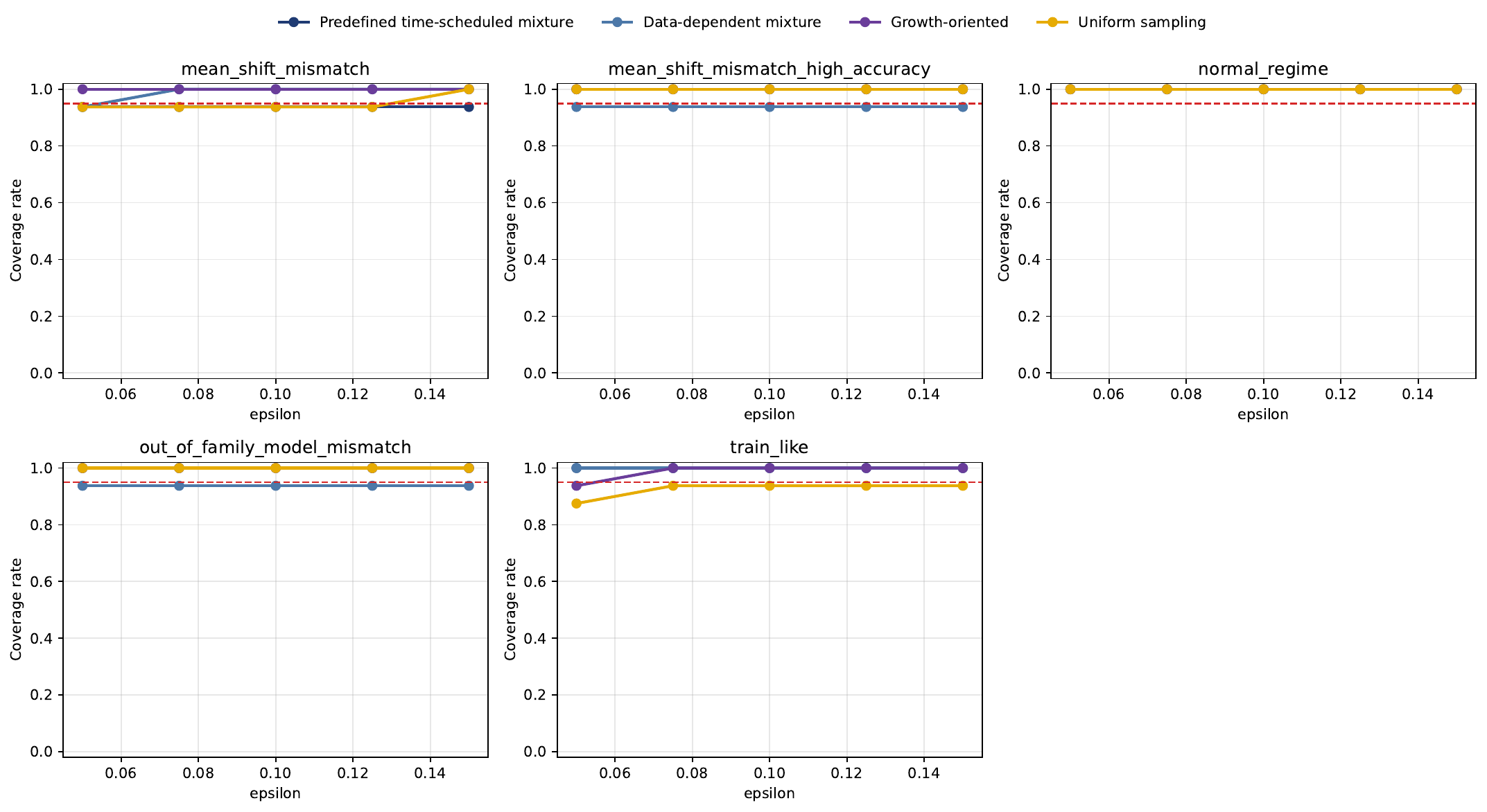}
    \caption{Practical setting: anytime coverage rate for each target accuracy $\epsilon$ for RIPr-based methods. For each value of $\epsilon$, we check whether $\theta_A^{(b)}\in C_t$ for each $b\in[B]$ holds over all query rounds until reaching the stopping criterion for each $\epsilon$ value ($|C_t|\leq\epsilon$). The reported value is the rate of $\mathbf{1}\{\forall\,t\leq\tau_\epsilon,\,\theta_A^{(b)}\in [L_t, U_t]\}=1$, where $\tau_\epsilon:=\inf\{t\ge1:U_t-L_t\leq\epsilon\}$, over $R\times B$ repetitions. If this rate is around $0.95$, then we say it is consistent with the level-$\alpha$ property, $\alpha=0.05$, considering Monte-Carlo error. The curves overlap in this figure.}
    \label{fig:ripr_only_coverage}
\end{figure}
\begin{figure}[t]
    \centering
    \includegraphics[width=0.95\textwidth]{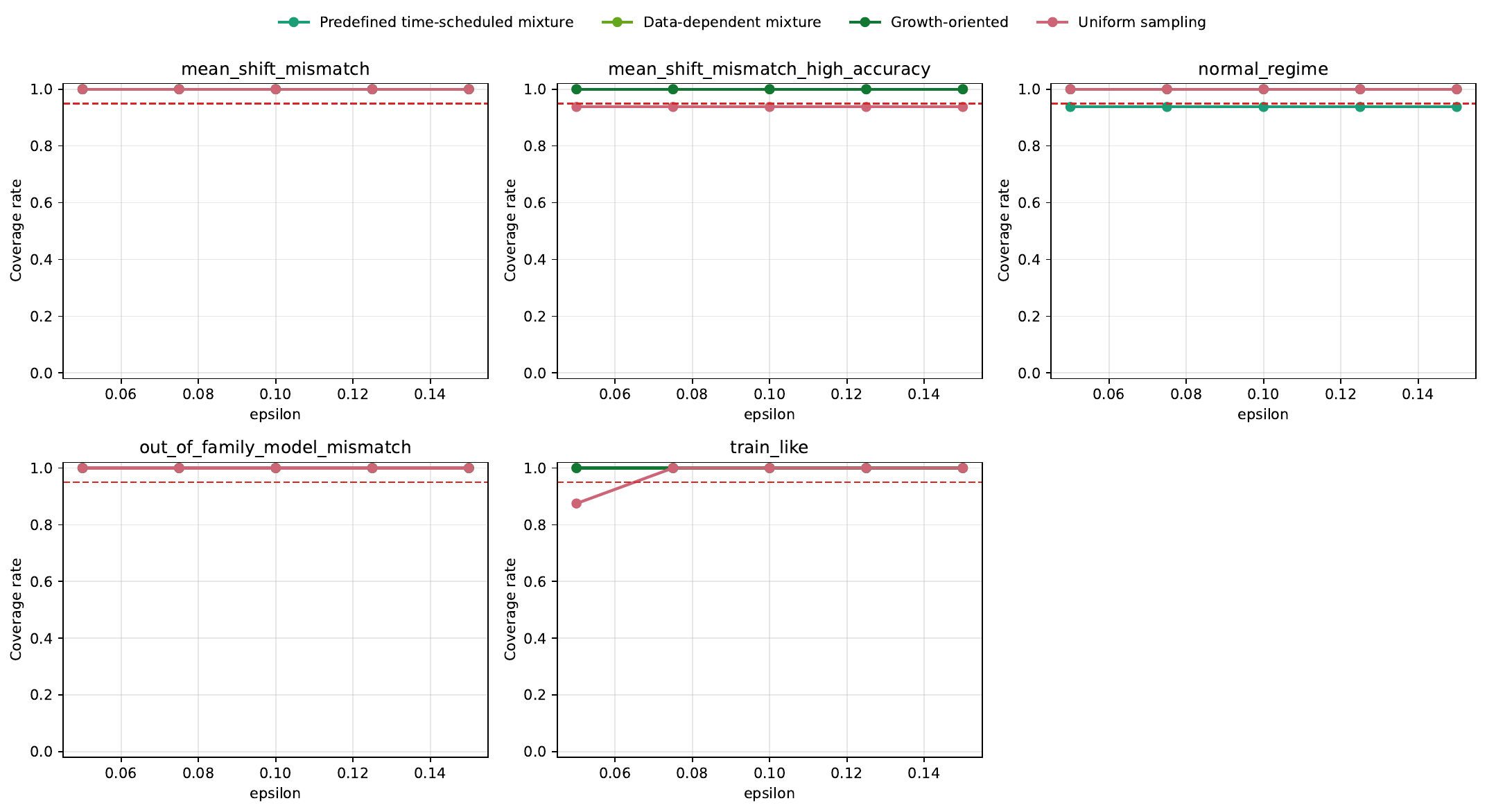}
    \caption{Practical setting: anytime coverage rate for each target accuracy $\epsilon$ for testing-by-betting-based CSs. For each value of $\epsilon$, we check whether $\theta_A^{(b)}\in C_t$ for each $b\in[B]$ holds over all query rounds until reaching the stopping criterion for each $\epsilon$ value ($|C_t|\leq\epsilon$). The reported value is the rate of $\mathbf{1}\{\forall\,t\leq\tau_\epsilon,\,\theta_A^{(b)}\in [L_t, U_t]\}=1$, where $\tau_\epsilon:=\inf\{t\ge1:U_t-L_t\leq\epsilon\}$, over $R\times B$ repetitions. If this rate is around $0.95$, then we say it is consistent with the level-$\alpha$ property, $\alpha=0.05$, considering Monte-Carlo error. The curves overlap in this figure.}
    \label{fig:bet_only_coverage}
\end{figure}

\begin{remark}
The proposed mixture querying rules are not unique, and other designs may perform better in certain capability regimes. Our goal is to provide two initial examples that show how growth, refinement, and uniform exploration can be combined, guided by the factors identified in our CS shrinkage analysis, thereby motivating future querying-rule designs.
\end{remark}

\subsection{Effect of the mismatch}\label{sec:exp_verify_plug-in_impact}
In this experiment, we isolate the effect of prediction mismatch and compare its impact on the RIPr-based and testing-by-betting-based CSs.

To control the mismatch, we use the same synthetic datasets as in Section~\ref{sec:exp_verify_oracle}, and, instead of using the Bayesian factor model from \cite{wuCandes2026efficientevaluationllmperformance}, we construct a noisy version of $p^\star$ as the plug-in prediction. Specifically, for each noise level $\sigma$, we set
\[
\hat p_t^{(\sigma)}(i)=\operatorname{clip}\left(p_i^\star+\xi_i^{(\sigma)},\epsilon,1-\epsilon\right),
\qquad
\xi_i^{(\sigma)}\sim \mathcal N(0,\sigma^2),
\qquad i\in[N],\ t\ge1.
\]
Therefore, we set the reference vector $r_t=\hat p_t^{(\sigma)}$ in Algorithm 1.
Here, the noise variables $\{\xi_i^{(\sigma)}\}_{i=1}^N$ are sampled once and then kept fixed over time to simplify the experiment. All other experimental settings are the same as those in Section~\ref{sec:exp_verify_oracle}. We repeat the experiment for $\sigma\in\{0, 0.1, 0.2, 0.3\}$. For each value of $\sigma$, we plot the average CS width over time, where the average is taken over all $B\times R$ independent runs. 

For the RIPr-based plug-in CS, the figure is shown in Figure~\ref{fig:plug-in_ripr_width_over_time}.
\begin{figure}[t]
    \centering
    \includegraphics[width=0.95\textwidth]{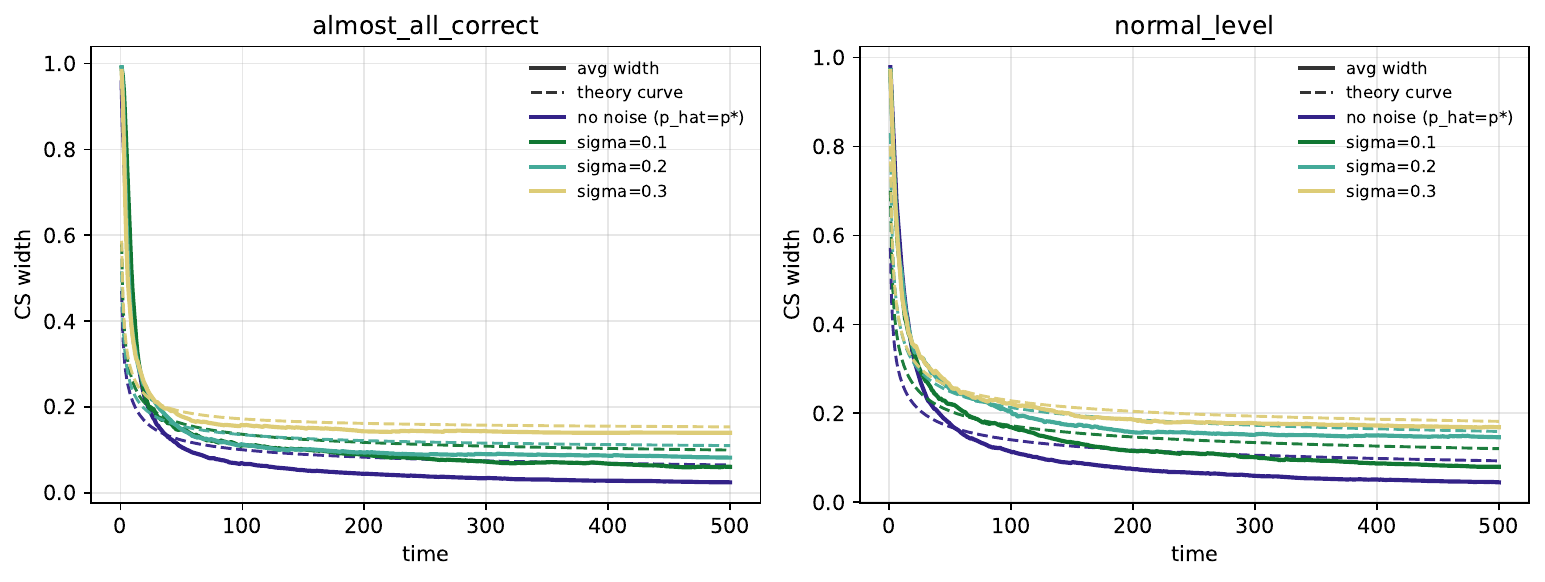}
    \caption{Plug-in RIPr-based CS: Average CS width over time for different values of $\sigma$. Larger values of $\sigma$ indicate larger prediction mismatch. In general, the CS width increases with $\sigma$, consistent with the behavior suggested by the theoretical curves.}
    \label{fig:plug-in_ripr_width_over_time}
\end{figure}
We scale the theoretical curve by finding a calibration constant to match the practical curve. In words, we find a constant $c$ for each theoretical curve such that $c = \arg\min_{c \ge 0} \sum_{t=1}^{N}(\texttt{avg\_width\_t} - c * \texttt{raw\_bound\_t})^2$.

For the betting-based plug-in CS, we observe a similar phenomenon. The difference is that, in this experiment, we set $q_t=0.4q_t^{\mathrm{method}}+0.6/N$. This is because, when the lower bound of $q_t$ is too small, i.e., when $h$ is too small, the valid betting range becomes narrow. As a result, the wealth process grows slowly, making it difficult to choose a constant that properly scales the theoretical bound to get an easily readable figure. The corresponding experimental results are shown in Figure~\ref{fig:plug-in_bet_width_over_time}.
\begin{figure}[t]
    \centering
    \includegraphics[width=0.95\textwidth]{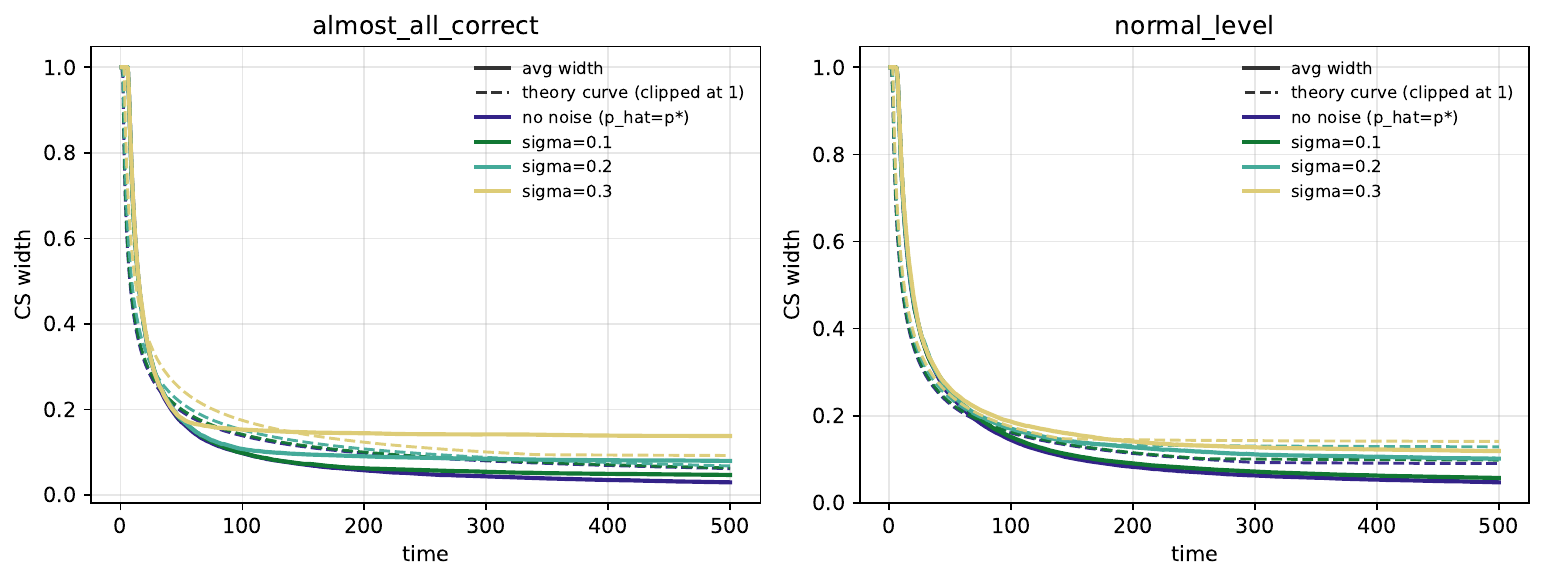}
    \caption{Plug-in testing-by-betting-based CS: Average CS width over time for different values of $\sigma$. Larger values of $\sigma$ indicate larger prediction mismatch. Compared with the RIPr-based CS in Figure~\ref{fig:plug-in_ripr_width_over_time}, the betting-based CS is less sensitive to increasing $\sigma$ in this experiment.}
    \label{fig:plug-in_bet_width_over_time}
\end{figure}

Then, we observe that, as $\sigma$ increases, the CS width deviates more substantially for the RIPr-based approach than for the testing-by-betting-based approach.

\section{Technical lemmas and theorems}

\begin{proposition}
\label{thm:quadratic_sampling_rule_and_bet}
Maximizing the quadratic lower bound \eqref{eq:exp_quadratic_lower_bound} over $\frac12\Lambda_h(m)$ yields the projected closed-form bet
\[
\lambda_t^{\mathrm{alt}}(m,q_t)
:=
\Pi_{\frac12\Lambda_h(m)}
\left\{
\frac{\theta^\star-m}{
2\left[
(\theta^\star-m)^2
+
\frac1{N^2}\sum_{i=1}^N
\frac{p_i^\star(1-p_i^\star)}{q_t(i)}
\right]}
\right\},
\]
where $\Pi_{\frac12\Lambda_h(m)}$ denotes the projection onto the domain $\frac12\Lambda_h(m)$. Moreover, if the unconstrained maximizer lies in $\frac12\Lambda_h(m)$, so that the projection is inactive, then after substituting $\lambda_t^{\mathrm{alt}}(m,q_t)$ into the quadratic lower bound in \eqref{eq:exp_quadratic_lower_bound}, the sampling distribution that maximizes this bound is $q_t^{\mathrm{alt}}(i)
\propto
\sqrt{p_i^\star(1-p_i^\star)}$ for all $i\in[N]$.
\end{proposition}
We omit the proof since it is a simple concave optimization problem.

\begin{proposition}[Computing RIPr point]\label{prop:solve_ripr_point}
    Let a reference vector $r\in(0,1)^N$. For any $m\in(0,1)$, there exists a unique $\lambda^\star(m)$ such that
\[
H(\lambda^\star(m))=m,
\]
where 
\[
H(\lambda)
:=
\frac1N\sum_{i=1}^N p_i(\lambda)\quad\text{and}\quad
p_i(\lambda)
=
\begin{cases}
r_i,
&
\lambda=0,
\\[6pt]
\displaystyle
\frac{
2q_t(i)r_i
}{
q_t(i)+\lambda
+
\sqrt{(q_t(i)+\lambda)^2-4\lambda q_t(i)r_i}
},
&
\lambda\ne0.
\end{cases}
\]
Then, for any $m\in(0,1)$, the RIPr point $p_{q_t}^\dagger(m;r)_i
=
p_i(\lambda^\star(m))$ for any $i\in[N]$. In practice, $\lambda^\star(m)$ can be found by a 1-dimensional root-finding algorithm on solving $H(\lambda^\star(m))=m$, e.g., the bisection method.
\end{proposition}

\begin{proof}
We recall how to compute the RIPr point at each round $t$. Fix a round $t$, a querying distribution $q_t\in\Delta_N$ with $q_t(i)>0$ for all $i\in[N]$, and a reference vector $r\in(0,1)^N$. In the oracle case, we take $r=p^\star$; in the plug-in case, we take $r=\hat p_{t-1}$. For a candidate value $m\in(0,1)$, the RIPr point of $Q_{r,q_t}$ projected on $\mathcal{P}_m$ is defined as
\[
p_{q_t}^\dagger(m;r)
\in
\arg\min_{p\in\mathcal P_m}
D(Q_{r,q_t}\|Q_{p,q_t}),
\]
where
\[
\mathcal P_m
=
\left\{
p\in[0,1]^N:
\frac1N\sum_{i=1}^N p_i=m
\right\}.
\]
Since the querying distribution is the same in $Q_{r,q_t}$ and $Q_{p,q_t}$, the KL divergence decomposes as
\[
D(Q_{r,q_t}\|Q_{p,q_t})
=
\sum_{i=1}^N q_t(i)d_{\mathrm{kl}}(r_i,p_i),
\]
where $d_{\mathrm{kl}}(a,b)
=
a\log\frac{a}{b}
+
(1-a)\log\frac{1-a}{1-b}$.

Therefore, computing the RIPr point is equivalent to solving
\[
p_{q_t}^\dagger(m;r)
\in
\arg\min_{p\in\mathcal P_m}
\sum_{i=1}^N q_t(i)d_{\mathrm{kl}}(r_i,p_i).
\]

We first claim that the optimizer is unique and lies in the interior. Since $r_i\in(0,1)$, the term $d_{\mathrm{kl}}(r_i,p_i)$ diverges when $p_i$ approaches $0$ or $1$. Hence, for $m\in(0,1)$, the optimizer must satisfy $p_i\in(0,1)$ for all $i\in[N]$. Moreover, for each coordinate,
\[
\frac{\partial^2}{\partial p_i^2}
d_{\mathrm{kl}}(r_i,p_i)
=
\frac{r_i}{p_i^2}
+
\frac{1-r_i}{(1-p_i)^2}
>
0.
\]
Thus, the objective is strictly convex in $p$, resulting in a unique RIPr point.

We now derive the first-order condition. By the Lagrangian method, we have
\[
\mathcal L(p,\lambda)
=
\sum_{i=1}^N q_t(i)d_{\mathrm{kl}}(r_i,p_i)
+
\lambda
\left(
\sum_{i=1}^N p_i-Nm
\right),
\]
where $\lambda\in\mathbb R$ is the Lagrange multiplier for the equality constraint. Since the optimizer is interior, it satisfies
\[
\frac{\partial}{\partial p_i}\mathcal L(p,\lambda)=0,
\qquad i\in[N].
\]
Therefore, through a simple algebra, the first-order condition becomes
\[
q_t(i)
\frac{p_i-r_i}{p_i(1-p_i)}
+
\lambda
=
0,
\qquad i\in[N].
\]

For a fixed value of $\lambda$, this equation determines each coordinate $p_i$. Multiplying both sides by $p_i(1-p_i)$ gives
\[
q_t(i)(p_i-r_i)
+
\lambda p_i(1-p_i)
=
0.
\]
After rearranging terms, $p_i$ satisfies
\[
\lambda p_i^2
-
(q_t(i)+\lambda)p_i
+
q_t(i)r_i
=
0.
\]
When $\lambda=0$, the first-order condition gives $p_i=r_i$. When $\lambda\ne0$, the root that is continuous at $\lambda=0$ is
\[
p_i(\lambda)
=
\frac{
q_t(i)+\lambda
-
\sqrt{(q_t(i)+\lambda)^2-4\lambda q_t(i)r_i}
}{
2\lambda
}.
\]
For numerical stability, we use the equivalent rationalized form
\[
p_i(\lambda)
=
\frac{
2q_t(i)r_i
}{
q_t(i)+\lambda
+
\sqrt{(q_t(i)+\lambda)^2-4\lambda q_t(i)r_i}
}.
\]
Thus, we define
\[
p_i(\lambda)
=
\begin{cases}
r_i,
&
\lambda=0,
\\[6pt]
\displaystyle
\frac{
2q_t(i)r_i
}{
q_t(i)+\lambda
+
\sqrt{(q_t(i)+\lambda)^2-4\lambda q_t(i)r_i}
},
&
\lambda\ne0.
\end{cases}
\]

It remains to choose $\lambda$ so that the mean constraint is satisfied. Define
\[
H(\lambda)
=
\frac1N\sum_{i=1}^N p_i(\lambda).
\]
We need to solve $H(\lambda)=m$. Then, we first claim that $H(\lambda)$ is continuous and strictly decreasing. To see this, define
\[
F_i(p,\lambda)
=
q_t(i)\frac{p-r_i}{p(1-p)}
+
\lambda.
\]
Since $F_i(p_i(\lambda),\lambda)=0$, implicit differentiation gives
\[
\frac{d p_i(\lambda)}{d\lambda}
=
-
\frac{1}{
q_t(i)
\left(
\frac{r_i}{p_i(\lambda)^2}
+
\frac{1-r_i}{(1-p_i(\lambda))^2}
\right)
}.
\]
The denominator is positive, and hence
\[
\frac{d p_i(\lambda)}{d\lambda}<0.
\]
Thus, every $p_i(\lambda)$ is strictly decreasing in $\lambda$, and therefore $H(\lambda)$ is strictly decreasing. Moreover, using the explicit expression of $p_i(\lambda)$, we have
\[
\lim_{\lambda\to\infty}p_i(\lambda)=0,
\qquad
\lim_{\lambda\to-\infty}p_i(\lambda)=1,
\qquad i\in[N].
\]
Indeed, when $\lambda\to\infty$, the denominator in
\[
p_i(\lambda)
=
\frac{2q_t(i)r_i}
{q_t(i)+\lambda+\sqrt{(q_t(i)+\lambda)^2-4\lambda q_t(i)r_i}}
\]
diverges to infinity, while the numerator is fixed. Hence $p_i(\lambda)\to0$. On the other hand, when $\lambda\to-\infty$, the same expression yields $p_i(\lambda)\to1$. Therefore,
\[
\lim_{\lambda\to\infty}H(\lambda)=0,
\qquad
\lim_{\lambda\to-\infty}H(\lambda)=1.
\]

Therefore, there exists a unique $\lambda^\star(m)$ such that $H(\lambda^\star(m))=m$.
Therefore, the RIPr point is
\[
p_{q_t}^\dagger(m;r)_i
=
p_i(\lambda^\star(m)),
\qquad i\in[N].
\]

In practice, $\lambda^\star(m)$ can be found by bisection because $H$ is continuous and strictly decreasing. Let
\[
\bar r
=
\frac1N\sum_{i=1}^N r_i.
\]
If $m=\bar r$, then $\lambda^\star(m)=0$ and $p_{q_t}^\dagger(m;r)=r$. If $m<\bar r$, then $\lambda^\star(m)>0$; if $m>\bar r$, then $\lambda^\star(m)<0$. Thus, one can search for $\lambda^\star(m)$ on the corresponding half-line and return $p_{q_t}^\dagger(m;r)_i=p_i(\lambda^\star(m))$ for all $i\in[N]$.

\begin{remark}[Endpoint candidates]
This proposition is stated for $m\in(0,1)$. For $m=0$ or $m=1$, the optimizer over the closed null class is the all-zero or all-one vector, respectively. Since $r\in(0,1)^N$, the corresponding RIPr e-value may be infinite for observations incompatible with the endpoint null. This does not affect e-value validity when interpreted in the extended sense, but it is inconvenient for computation. Therefore, in practice, we restrict the numerical grid to the interior of $[0,1]$. 
\end{remark}

\end{proof}

\begin{lemma}[Product of conditional e-values is a test supermartingale]\label{lem:product_of_e-val_is_e-process}
    Let $E_t$ be a conditional e-value for a distribution class $\mathcal{P}$, i.e., $\mathbb{E}_p[E_t|\mathcal{F}_{t-1}]\leq1$ and $E_t\geq0$ for all $t\ge1$ almost surely for all $p\in\mathcal{P}$. Then, $\{W_t\}_{t\ge0}$ is a test supermaringale for $\mathcal{P}$ if $W_t:=\Pi_{s=1}^{t}E_s$ and $W_0\leq1$.
\end{lemma}
\begin{proof}

    This proof is included for completeness. The proof is simple. Since $\mathbb{E}_p[W_t|\mathcal{F}_{t-1}]=\mathbb{E}_p[\Pi_{s=1}^{t}E_s|\mathcal{F}_{t-1}]=\Pi_{s=1}^{t-1} E_s\mathbb{E}_p[E_t|\mathcal{F}_{t-1}]\leq \Pi_{s=1}^{t-1}E_s=W_{t-1}$ for all $p\in\mathcal{P}$ and $W_t\ge0$ for any $t\ge0$, we get $\{W_t\}$ is a non-negative supermartingale for $\mathcal{P}$. In \cite{WaudbySmith2024BoundedBetting} and \cite{Ramdas_2025}, they also called such $\{W_t\}$ as ``test supermartingale". 
\end{proof}

\begin{lemma}[Freedman's inequality \cite{freedman1975tail,sfyraki2026}]\label{lem:freeman}
    Let $\zeta_1,\ldots,\zeta_T$ be a martingale difference sequence with a uniform upper bound
$|\zeta_t|\le b$ for all $t$. Denote $V_t$ as the sum of conditional variances of
$\zeta_s$'s, i.e., $V_t=\sum_{s=1}^t \mathrm{Var}(\zeta_s\mid \zeta_1,\ldots,\zeta_{s-1})$.
Also, denote $\sigma_t:=\sqrt{V_t}$. Then, for any $0<\delta<1/e$ and $T\ge 4$,
we have
\[
\mathbb{P}\left(
\exists t\le T:
\sum_{s=1}^t \zeta_s
\ge
2\max\left\{
2\sigma_t,\,
b\sqrt{\log(1/\delta)}
\right\}
\sqrt{\log(1/\delta)}
\right)
\le
\log(T)\delta.
\]
\end{lemma}

\begin{lemma}[Bound for a general CS width on a finite grid]
\label{lem:bound_for_general_CS_width}
Under the Bernoulli setting for the question correctness in
Section~\ref{sec:ripr_based_oracle_design}, let $E_s(m)>0$ be the
conditional e-value for $\mathcal P_m$ under a predictable selection rule.
Fix a finite grid
\[
\mathcal M_K=\{m_1,\ldots,m_K\}\subseteq[0,1],
\qquad m_1=0,\quad m_K=1.
\]
For each $m\in\mathcal M_K$, define
\[
\zeta_s(m)
:=
\log E_s(m)
-
\mathbb{E}_{Q_{p^\star,q_s}}
\left[
\log E_s(m)\mid\mathcal F_{s-1}
\right],
\qquad s\ge 1.
\]
Assume that $|\zeta_s(m)|\le b$ for all $m\in\mathcal M_K$ and $s\le N$.
Define
\[
V_t(m)
:=
\sum_{s=1}^t
\mathrm{Var}_{Q_{p^\star,q_s}}
\left(
\zeta_s(m)\mid \mathcal F_{s-1}
\right),
\qquad
V_t:=\sup_{m\in\mathcal M_K}V_t(m),
\qquad
\sigma_t:=\sqrt{V_t}.
\]
Then, for any $\delta\in(0,1)$ and $N\ge 4$, with probability at least
$1-\delta$, we have, for all $t\in[N]$,
\[
C_t\cap\mathcal M_K\subseteq C_t(\delta),
\]
where $C_t$ is the confidence sequence constructed by $E_s(m)$ through
Section~\ref{sec:prelim}, and
\[
C_t(\delta)
:=
\left\{
m\in\mathcal M_K:
\sum_{s=1}^{t}
\mathbb{E}_{Q_{p^\star,q_s}}
\left[
\log E_s(m)\mid\mathcal F_{s-1}
\right]
-
\mathrm{pen}_t(\delta)
\le
\log\frac{1}{\alpha}
\right\},
\]
with
\[
\mathrm{pen}_t(\delta)
:=
2\max\left\{
2\sigma_t,\,
b\sqrt{\log\left(\frac{K\log N}{\delta}\right)}
\right\}
\sqrt{\log\left(\frac{K\log N}{\delta}\right)}.
\]
\end{lemma}

\begin{proof}
For each $m\in\mathcal M_K$, by definition,
\[
\mathbb{E}_{Q_{p^\star,q_s}}
\left[
\zeta_s(m)\mid\mathcal F_{s-1}
\right]
=0.
\]
Hence $\{\zeta_s(m)\}_{s=1}^N$ is a martingale difference sequence under
the true distribution induced by $Q_{p^\star,q_s}$ and the predictable
selection rule.

Fix $m\in\mathcal M_K$. Applying Lemma~\ref{lem:freeman} to the martingale
difference sequence $\{-\zeta_s(m)\}_{s=1}^N$ with
\[
\eta=\frac{\delta}{K\log N},
\]
we obtain
\[
\mathbb P\left(
\exists t\le N:
-\sum_{s=1}^t\zeta_s(m)
\ge
2\max\left\{
2\sqrt{V_t(m)},\,
b\sqrt{\log\left(\frac{K\log N}{\delta}\right)}
\right\}
\sqrt{\log\left(\frac{K\log N}{\delta}\right)}
\right)
\le
\frac{\delta}{K}.
\]
Since $V_t(m)\le V_t$ and $\sigma_t=\sqrt{V_t}$, we further have
\[
\mathbb P\left(
\exists t\le N:
-\sum_{s=1}^t\zeta_s(m)
\ge
\mathrm{pen}_t(\delta)
\right)
\le
\frac{\delta}{K}.
\]
Taking a union bound over $m\in\mathcal M_K$, with probability at least
$1-\delta$, the following event holds:
\[
\mathcal B_\delta
:=
\left\{
\sum_{s=1}^t\zeta_s(m)
\ge
-\mathrm{pen}_t(\delta),
\quad
\forall t\le N,\ \forall m\in\mathcal M_K
\right\}.
\]

Now work on the event $\mathcal B_\delta$. For any $t\le N$ and any
$m\in C_t\cap\mathcal M_K$, by the definition of $C_t$,
\[
\prod_{s=1}^t E_s(m)\le \frac1\alpha.
\]
Equivalently,
\[
\sum_{s=1}^t \log E_s(m)
\le
\log\frac1\alpha.
\]
Using the definition of $\zeta_s(m)$, we have
\[
\sum_{s=1}^t \log E_s(m)
=
\sum_{s=1}^{t}
\mathbb{E}_{Q_{p^\star,q_s}}
\left[
\log E_s(m)\mid\mathcal F_{s-1}
\right]
+
\sum_{s=1}^t\zeta_s(m).
\]
On $\mathcal B_\delta$,
\[
\sum_{s=1}^t\zeta_s(m)
\ge
-\mathrm{pen}_t(\delta).
\]
Therefore,
\[
\sum_{s=1}^{t}
\mathbb{E}_{Q_{p^\star,q_s}}
\left[
\log E_s(m)\mid\mathcal F_{s-1}
\right]
-
\mathrm{pen}_t(\delta)
\le
\sum_{s=1}^t \log E_s(m)
\le
\log\frac1\alpha.
\]
Hence $m\in C_t(\delta)$. Since this argument holds for every
$m\in C_t\cap\mathcal M_K$ and every $t\le N$ on $\mathcal B_\delta$, we get
\[
C_t\cap\mathcal M_K\subseteq C_t(\delta),
\qquad \forall t\le N.
\]
Since $\mathbb P(\mathcal B_\delta)\ge 1-\delta$, the result follows.
\end{proof}

\begin{lemma}[Curvature lower bound for oracle RIPr growth]
\label{lem:oracle_ripr_curvature}
Under the Bernoulli setting for the question correctness, suppose that
\[
\mathcal P_m
=
\left\{
p\in[0,1]^N:
\frac1N\sum_{i=1}^N p_i=m
\right\},
\qquad
\theta^\star=\frac1N\sum_{i=1}^N p_i^\star.
\]
For any predictable querying distribution $q_t$ with $q_t(i)>0$ for all
$i\in[N]$, the oracle RIPr one-step expected log-growth satisfies $G_t(q_t, E_t^{\mathrm{RIPr}}(m, p^\star,q_t),p^\star)
\ge
c_t(q_t)(\theta^\star-m)^2$,
where $c_t(q_t)
:=
\frac{2N^2}{\sum_{i=1}^N 1/q_t(i)}$.
\end{lemma}

\begin{proof}
By the definition of the oracle RIPr factor,
\[
G_t(q_t, E_t^{\mathrm{RIPr}}(m, p^\star,q_t),p^\star)
=
\sum_{i=1}^N
q_t(i)
D_{\mathrm{KL}}
\left(
\mathrm{Bern}(p_i^\star)
\middle\|
\mathrm{Bern}(p_{q_t}^\dagger(m;p^\star)_i)
\right).
\]
By Pinsker's inequality for Bernoulli distributions, $D_{\mathrm{KL}}
\left(
\mathrm{Bern}(p)
\middle\|
\mathrm{Bern}(q)
\right)
\ge
2(p-q)^2$.
Hence
\[
G_t(q_t, E_t^{\mathrm{RIPr}}(m, p^\star,q_t),p^\star)
\ge
2\sum_{i=1}^N
q_t(i)
\left(
p_i^\star-p_{q_t}^\dagger(m)_i
\right)^2.
\]
Let $a_i=p_i^\star-p_{q_t}^\dagger(m;p^\star)_i$.
Since $p_{q_t}^\dagger(m;p^\star)\in\mathcal P_m$, we have
\[
\sum_{i=1}^N a_i
=
N(\theta^\star-m).
\]
By Cauchy-Schwarz and $q_t(i)>0$,
\[
\left(\sum_{i=1}^N a_i\right)^2
=\left(\sum_{i=1}^{N}a_i\frac{1}{\sqrt{q_t(i)}}\sqrt{q_t(i)}\right)^2\le
\left(
\sum_{i=1}^N q_t(i)a_i^2
\right)
\left(
\sum_{i=1}^N \frac1{q_t(i)}
\right).
\]
Therefore,
\[
\sum_{i=1}^N q_t(i)a_i^2
\ge
\frac{N^2(\theta^\star-m)^2}
{\sum_{i=1}^N 1/q_t(i)}.
\]
Combining the above inequalities yields
\[
G_t(q_t, E_t^{\mathrm{RIPr}}(m, p^\star,q_t),p^\star)
\ge
\frac{2N^2}{\sum_{i=1}^N 1/q_t(i)}
(\theta^\star-m)^2.
\]
This proves the claim.
\end{proof}

\begin{lemma}[Mismatch in RIPr]\label{lem:mismatch_in_ripr}
    Under the Bernoulli setting for the question correctness in Section~\ref{sec:ripr_based_oracle_design} and assuming $q_t(i)>0$ and $p_i^\star,\hat p_{t-1}(i)\in(0,1)$ for all $i$, we have, for any $m\in[0,1]$, $t\ge1$,
    \begin{align}
         0\leq G_t(q_t, E_t^{\mathrm{RIPr}}(m, p^\star,q_t),p^\star) - G_t(q_t, E_t^{\mathrm{RIPr}}(m,\hat p_{t-1},q_t),p^\star)\leq D(Q_{p^\star,q_t}\|Q_{\hat p_{t-1},q_t}),
    \end{align}
    where $E_t^{\mathrm{RIPr}}(\cdot)$ is defined in Proposition~\ref{prop:ripr_validity}.
\end{lemma}
\begin{proof}
We only prove the upper bound since the lower bound is from the oracle optimality of RIPr e-value in \eqref{eq:ripr_oracle_optimality}. Fix $t\ge1$, $m\in[0,1]$, and condition on $\mathcal F_{t-1}$. Then $q_t$ and $\hat p_{t-1}$ are fixed. For notational simplicity, write
\[
P:=Q_{p^\star,q_t},
\qquad
\hat P:=Q_{\hat p_{t-1},q_t}.
\]
Let $P_m^\dagger:=Q_{p_{q_t}^\dagger(m;p^\star),q_t}$ be the RIPr point of $P$ onto the null class $\mathcal P_m$, and let $\hat P_m^\dagger:=Q_{\hat p_{q_t}^\dagger(m;\hat p_{t-1}),q_t}$ be the RIPr point of $\hat P$ onto the same null class. For simplicity, we suppress the conditional expectation expression. By definition,
\[
 G_t(q_t, E_t^{\mathrm{RIPr}}(m, p^\star,q_t),p^\star)
=
\mathbb E_P\left[\log\frac{dP}{dP_m^\dagger}\right],
\]
and the plug-in RIPr e-value has expected log-growth
\[
G_t(q_t,E_t^{\mathrm{RIPr}}(m,\hat p_{t-1},q_t),p^\star)
=
\mathbb E_P\left[\log\frac{d\hat P}{d\hat P_m^\dagger}\right].
\]
Therefore,
\begin{align}
&G_t(q_t, E_t^{\mathrm{RIPr}}(m, p^\star,q_t),p^\star)
-
G_t(q_t,E_t^{\mathrm{RIPr}}(m,\hat p_{t-1},q_t),p^\star)\\
&=
\mathbb E_P\left[
\log\frac{dP}{dP_m^\dagger}
-
\log\frac{d\hat P}{d\hat P_m^\dagger}
\right] \\
&=
\mathbb E_P\left[\log\frac{dP}{d\hat P}\right]
+
\mathbb E_P\left[\log\frac{d\hat P_m^\dagger}{dP_m^\dagger}\right] \\
&=
D(P\|\hat P)
+
\mathbb E_P\left[\log\frac{d\hat P_m^\dagger}{dP_m^\dagger}\right].
\end{align}

It remains to show that the second term is smaller than or equal to zero. Since $P_m^\dagger$ is the reverse information projection of $P$ onto the null class, it minimizes $D(P\|Q)$ over all $Q\in\mathcal P_m$. Since $\hat P_m^\dagger\in\mathcal P_m$, the optimality of $P_m^\dagger$ implies
\[
D(P\|\hat P_m^\dagger)
\ge
D(P\|P_m^\dagger).
\]
Equivalently, we have
\[
\mathbb E_P\left[\log\frac{dP}{d\hat P_m^\dagger}\right]
\ge
\mathbb E_P\left[\log\frac{dP}{dP_m^\dagger}\right],
\]
which gives $\mathbb E_P\left[\log\frac{d\hat P_m^\dagger}{dP_m^\dagger}\right]\le 0$.
Hence,
\[
G_t(q_t, E_t^{\mathrm{RIPr}}(m, p^\star,q_t),p^\star)
-
G_t(q_t,E_t^{\mathrm{RIPr}}(m,\hat p_{t-1},q_t),p^\star)
\le
D(P\|\hat P)=D(Q_{p^\star,q_t}\|Q_{\hat p_{t-1},q_t}).
\]
This proves the desired upper bound.
\end{proof}

\section{Proof of main theorems}

\subsection{Proof of Proposition~\ref{prop:ripr_validity}}\label{proof:proof_of_exp_log-wealth_opt_fix_query_strategy}
\begin{proof}
This proof is included for completeness. Let us first prove the validity of the RIPr one-factor value is a conditional e-value.
The positivity of $E^{\mathrm{RIPr}}_t(m)$ is trivial. We only show the expectation condition. For notational simplicity, we denote $E_t^{\mathrm{RIPr}}(m, p^\star, q_t)$ as $E_t^{\mathrm{RIPr}}(m)$ and $\mathbb{E}_{(I,Z_{I})\sim Q_{p,q}}[\cdot\,|\,\mathcal{F}_{t-1}]$ as $\mathbb{E}_{Q_{p,q}}[\cdot]$.

\paragraph{For $m\in(0,1)$.}
Under a fixed $\mathcal{F}_{t-1}$, $q_t$ is fixed. Hence, for notational simplicity, we suppress the conditioning on $\mathcal{F}_{t-1}$, let $q\equiv q_t$, and $p^\dagger_{q_t}(m;p^\star)\equiv p^\dagger$. By simple algebra, we have
\begin{align}
\mathbb E_{Q_{p,q}}\!\left[E_t^{\mathrm{RIPr}}(m)\right]-1
&=
\sum_{i=1}^N q(i)
\sum_{z\in\{0,1\}}
p_i^z(1-p_i)^{1-z}
\frac{
(p_i^\star)^z(1-p_i^\star)^{1-z}
}{
(p_i^\dagger)^z(1-p_i^\dagger)^{1-z}
} -1 \\
&=
\sum_{i=1}^N q(i)
\left[
p_i\frac{p_i^\star}{p_i^\dagger}
+
(1-p_i)\frac{1-p_i^\star}{1-p_i^\dagger}-1
\right]\\
&=\sum_{i=1}^N q(i)
\frac{(p_i-p_i^\dagger)(p_i^\star-p_i^\dagger)}
{p_i^\dagger(1-p_i^\dagger)}.
\end{align}
Then, we want to show $\sum_{i=1}^N q(i)
\frac{(p_i-p_i^\dagger)(p_i^\star-p_i^\dagger)}
{p_i^\dagger(1-p_i^\dagger)}\leq 0$.

By the RIPr point $p^\dagger
\in\arg\min_{p\in\mathcal P_m}
\sum_{i=1}^{N} q_t(i)d_{\mathrm{kl}}(p_i^\star,p_i)$, we know it is well-defined under the condition stated in Proposition~\ref{prop:ripr_validity} and $m\in(0,1)$ in this case. Hence, by standard finite-dimensional convex analytic techniques, we can get a clean characterization of the oracle e-value in terms of the RIPr.

The Lagrangian can be written as, for some $\mu\in\mathbb{R}$ and $a_i, b_i\in[0,\infty)$,
\[
\mathcal L(p,\mu,a,b)
=
\sum_{i=1}^N q(i) d_{\mathrm{kl}}(p_i^\star,p_i)
+
\mu\left(\sum_{i=1}^N p_i-Nm\right)
-
\sum_{i=1}^N a_ip_i
+
\sum_{i=1}^N b_i(p_i-1).
\]
On taking the derivative w.r.t $p_i$, the RIPr point $p^\dagger$ must satisfy the stationarity condition:
\[
q(i)
\frac{p_i^\dagger-p_i^\star} {p_i^\dagger(1-p_i^\dagger)} +\mu - a_i + b_i = 0, \qquad \forall i\in[N].
\]
Moreover, since $\min_{i\in[N]}p^\star_i>0$ and $\max_{i\in[N]}p^\star_i<1$ implies that $\min_{i\in[N]}p^\dagger_i>0$ and $\max_{i\in[N]}p^\dagger_i<1$, respectively (or the minimization is infinite). Therefore, by complementary slackness and the fact that $\min_{i\in[N]}p^\dagger_i>0$ and $\max_{i\in[N]}p^\dagger_i<1$, 
\[
a_i p_i^\dagger=0,
\qquad
b_i(p_i^\dagger-1)=0,
\qquad \forall i\in[N],
\]
we know $a_i=b_i=0$. 

Consequently, $p^\dagger$ must satisfy $q(i)
\frac{p_i^\dagger-p_i^\star} {p_i^\dagger(1-p_i^\dagger)} +\mu=0$,  $\forall i\in[N]$.
By this property, for any $p\in\mathcal{P}_m$, we get
\begin{align}
    \mathbb E_{Q_{p,q}}\!\left[E_t^{\mathrm{RIPr}}(m)\right]-1 &= \sum_{i=1}^N q(i)
\frac{(p_i-p_i^\dagger)(p_i^\star-p_i^\dagger)}
{p_i^\dagger(1-p_i^\dagger)}\\
&=-\mu\sum_{i=1}^{N}(p_i-p_i^\dagger)=0,
\end{align}
thereby, $\mathbb E_{Q_{p,q}}\!\left[E_t^{\mathrm{RIPr}}(m)\right]=1$, completing the proof.

Note that the proof only requires each $p^\star_i$ to lie in the open interval $(0,1)$ and to be $\mathcal{F}_{t-1}$-measurable at round $t$. Therefore, the same argument continues to hold if $p^\star$ is replaced by any predictable vector $r_t\in(0,1)^N$. In particular, the plug-in construction in Section~\ref{sec:prac_e-values}, which replaces $p^\star$ by the prediction $\hat p_{t-1}$, still yields valid conditional e-values as long as $\hat p_{t-1}(i)\in(0,1)$ for all $i\in[N]$, where this is true in our experimental implementations.

\paragraph{For $m=0$ or $m=1$.}
For the endpoint candidates, the RIPr point is defined by a limiting argument. When $m=0$, the null class reduces to the singleton $(0,\ldots,0)$, and we set $p_q^\dagger(0;p^\star)=(0,\ldots,0)$. The corresponding RIPr factor is interpreted as an extended nonnegative random variable:
\[
E_t^{\mathrm{RIPr}}(0,p^\star,q_t)
=
\begin{cases}
1-p^\star_{I_t}, & Z_{I_t}=0,\\
+\infty, & Z_{I_t}=1.
\end{cases}
\]
Similarly, when $m=1$, we set $p_q^\dagger(1;p^\star)=(1,\ldots,1)$ and define
\[
E_t^{\mathrm{RIPr}}(1,p^\star,q_t)
=
\begin{cases}
p^\star_{I_t}, & Z_{I_t}=1,\\
+\infty, & Z_{I_t}=0.
\end{cases}
\]
These extended-valued factors remain valid conditional e-values, because the infinite values occur only on events that have probability zero under the corresponding endpoint null.

Then, we also show the oracle optimality results for completeness. Fix $t\geq1$ and condition on $\mathcal F_{t-1}$. Since $q_t$ is predictable, it is fixed conditional on $\mathcal F_{t-1}$. For notational simplicity, write
\[
P:=Q_{p^\star,q_t},
\qquad
Q^\dagger:=Q_{p^\dagger_{q_t}(m;p^\star),q_t}.
\]
By the RIPr supporting condition, the likelihood ratio
\[
E_t^{\mathrm{RIPr}}(m)
=
\frac{dP}{dQ^\dagger}(I_t,Z_{I_t})
\]
is a valid conditional e-value for the composite null $\mathcal P_m$.

Let $E_t(m)\in\mathcal E_m(q_t)$ be any other conditional e-value. Since $Q^\dagger$ belongs to the null class induced by $\mathcal P_m$, the conditional e-value property gives
\[
\mathbb E_{Q^\dagger}
\left[
E_t(m)\mid \mathcal F_{t-1}
\right]
\leq 1.
\]
Then,
\[
\begin{aligned}
\mathbb E_{P}
\left[
\log E_t(m)\mid \mathcal F_{t-1}
\right]
&=
\mathbb E_{P}
\left[
\log\left(
E_t(m)\frac{dQ^\dagger}{dP}
\right)
\mid \mathcal F_{t-1}
\right]
+
\mathbb E_P
\left[
\log \frac{dP}{dQ^\dagger}
\mid \mathcal F_{t-1}
\right] \\
&\leq
\log
\mathbb E_{P}
\left[
E_t(m)\frac{dQ^\dagger}{dP}
\mid \mathcal F_{t-1}
\right]
+
D_{\mathrm{KL}}(P\|Q^\dagger) \\
&=
\log
\mathbb E_{Q^\dagger}
\left[
E_t(m)
\mid \mathcal F_{t-1}
\right]
+
D_{\mathrm{KL}}(P\|Q^\dagger) \\
&\leq
D_{\mathrm{KL}}(P\|Q^\dagger).
\end{aligned}
\]
The first inequality follows from Jensen's inequality. On the other hand, for the RIPr e-value,
\[
\begin{aligned}
\mathbb E_P
\left[
\log E_t^{\mathrm{RIPr}}(m)
\mid \mathcal F_{t-1}
\right]
&=
\mathbb E_P
\left[
\log \frac{dP}{dQ^\dagger}
\mid \mathcal F_{t-1}
\right] \\
&=
D_{\mathrm{KL}}(P\|Q^\dagger).
\end{aligned}
\]
Therefore,
\[
G(q_t,E_t^{\mathrm{RIPr}}(m),m,p^\star)
\geq
G(q_t,E_t(m),m,p^\star)
\qquad \text{a.s.}
\]

\end{proof}

\subsection{Proof of Proposition~\ref{prop:test_by_bet_validity}}\label{proof:test_by_bet_validity}
\begin{proof}
Fix $t\geq1$ and condition on $\mathcal F_{t-1}$. Since $q_t$ and $\lambda_t(m)$ are predictable, they are fixed conditional on $\mathcal F_{t-1}$. The condition $\min_i q_t(i)>0$ ensures that $\Phi_t\equiv\Phi_t(p^\star, q_t)$ is well-defined and finite for every possible observation.

First, by the definition of $\Lambda_{h_t}(m)$, for every possible realization $(I_t,Z_{I_t})=(i,z)$,
\[
1+\lambda_t(m)\left(
\frac{1}{N}\sum_{j=1}^{N}p_j^\star
+
\frac{z-p_i^\star}{Nq_t(i)}
-m
\right)
\geq 0.
\]
Therefore, $E_t^{\mathrm{bet}}(m)\geq0$ almost surely.

Then, we compute its conditional expectation under any $p\in\mathcal P_m$. Under $Q_{p,q_t}$, we have
\begin{align}
\mathbb{E}_{Q_{p,q_t}}
\left[
\Phi_t
\,\middle|\,
\mathcal F_{t-1}
\right]
&=
\frac{1}{N}\sum_{i=1}^{N}p_i^\star
+
\frac{1}{N}
\sum_{i=1}^{N}
q_t(i)
\frac{
\mathbb E[Z_{I_t}\mid I_t=i]-p_i^\star
}{
q_t(i)
} \\
&=
\frac{1}{N}\sum_{i=1}^{N}p_i^\star
+
\frac{1}{N}
\sum_{i=1}^{N}
(p_i-p_i^\star) \\
&=
\frac{1}{N}\sum_{i=1}^{N}p_i.
\end{align}
Since $p\in\mathcal P_m$, we have
\[
\frac{1}{N}\sum_{i=1}^{N}p_i=m.
\]
Therefore, $\mathbb{E}_{Q_{p,q_t}}
\left[
\Phi_t-m
\,\middle|\,
\mathcal F_{t-1}
\right]
=
0$, i.e. $\Phi_t$ is unbiased for $p$ with any predictable $\{q_t\}$.
Using that $\lambda_t(m)$ is $\mathcal F_{t-1}$-measurable, we get
\begin{align}
\mathbb{E}_{Q_{p,q_t}}
\left[
E_t^{\mathrm{bet}}(m)
\,\middle|\,
\mathcal F_{t-1}
\right]
&=
\mathbb{E}_{Q_{p,q_t}}
\left[
1+\lambda_t(m)(\Phi_t-m)
\,\middle|\,
\mathcal F_{t-1}
\right] \\
&=
1+\lambda_t(m)
\mathbb{E}_{Q_{p,q_t}}
\left[
\Phi_t-m
\,\middle|\,
\mathcal F_{t-1}
\right] \\
&=
1.
\end{align}
Thus, $\mathbb{E}_{(I,Z_I)\sim Q_{p,q_t}}
\left[
E_t^{\mathrm{bet}}(m)
\,\middle|\,
\mathcal F_{t-1}
\right]
\leq1$,
for every $p\in\mathcal P_m$. This completes the proof.

Note that the proof only requires each $p^\star_i$ to be $\mathcal{F}_{t-1}$-measurable at round $t$. Again, we can readily verify that the plug-in construction in Section~\ref{sec:prac_e-values}, which replaces $p^\star$ by the prediction $\hat p_{t-1}$, still yields valid conditional e-values since it is predictable.
\end{proof}

\subsection{Proof of Theorem~\ref{thm:mismatch_impact_for_RIPr}}\label{proof:proof_of_mismatch_impact_for_RIPr}
We start by providing the formal statement of the theorem.
\begin{theorem}[(Formal) Approximated width of RIPr approach]
For any $t\ge1$, suppose each element of the prediction $\hat p_{t-1}$ and the RIPr point of $Q_{\hat p_{t-1},q_t}$ projected on the null $\mathcal{P}_m$, $p^\dagger_{q_t}(m;\hat p_{t-1})$, is in the range $(\vartheta,1-\vartheta)$ for any $t\ge1$ for some $\vartheta>0$, the querying rule satisfies $q_t(i)\geq h$ for some $h\in(0,1/N]$ for any $i\in[N]$. With probability at least probability $1-\delta$, for any $t\in[N]$,  
\begin{align}
    |C_t^{\mathrm{RIPr}}\cap\mathcal{M}_K|&\leq2\sqrt{\frac{2\max\left\{2\sigma^{\mathrm{RIPr}}_t,2L_\vartheta\sqrt{\ell_{K,N,\delta}}\right\}\sqrt{\ell_{K,N,\delta}}+\sum_{s=1}^{t}D(Q_{p^\star,q_s}\|Q_{\hat p_{s-1}, q_s})+\log(1/\alpha)}{t2Nh}}.
\end{align}
where $\mathcal{M}_K$ is the set of $K$ grids, $\sigma^{\mathrm{RIPr}}_t:=\sqrt{\sup_{m\in\mathcal{M}_K}\sum_{s=1}^{t}\mathrm{Var}(\log E_s^{\mathrm{RIPr}}(\cdot)\,|\,\mathcal{F}_{t-1})}$, $L_\vartheta:=\log\frac{1-\vartheta}{\vartheta}$ and $\ell_{K,N,\delta}:=\log\left(\frac{K\log N}{\delta}\right)$.
\end{theorem}
The proof is as follows.

We apply Lemma~\ref{lem:bound_for_general_CS_width} to derive this result. We define 
\[
\zeta_t(m):=\log E^\mathrm{RIPr}_t(m, \hat p_{t-1}, q_t)-\mathbb{E}_{Q_{p^\star,q_t}}[\log E^\mathrm{RIPr}_t(m, \hat p_{t-1}, q_t)\,|\,\mathcal{F}_{t-1}] \qquad\forall\,t\ge1,
\]
where $E^\mathrm{RIPr}_t(m, \hat p_{t-1}, q_t):=\frac{
(\hat p_{t-1}(I_t))^{Z_{I_t}}(1-\hat p_{t-1}(I_t))^{1-Z_{I_t}}
}{
(p_{q_t}^\dagger(m;\hat p_{t-1})_{I_t})^{Z_{I_t}}(1-p_{q_t}^\dagger(m;\hat p_{t-1})_{I_t})^{1-Z_{I_t}}}$ and 
\[
p_{q_t}^\dagger(m;\hat p_{t-1})\in
\arg\min_{p\in\mathcal P_m}
D_{\mathrm{KL}}\left(Q_{\hat p_{t-1},q_t}\,\Vert\,Q_{p,q_t}\right)=
\arg\min_{p\in\mathcal P_m}
\sum_{i=1}^N q_t(i)d_{\mathrm{kl}}(\hat p_{t-1}(i),p_i).
\]
Then, we proceed to show the corresponding $\sigma_t$ and $b$.

Define
\[
Y_t(m):=\log E_t^{\mathrm{RIPr}}(m,\hat p_{t-1},q_t).
\]
For each $i\in[N]$, let
\[
A_{t,i}(m):=
\log\frac{\hat p_{t-1}(i)}
{\hat p_{q_t}^{\dagger}(m;\hat p_{t-1})_i},
\qquad
B_{t,i}(m):=
\log\frac{1-\hat p_{t-1}(i)}
{1-\hat p_{q_t}^{\dagger}(m;\hat p_{t-1})_i}.
\]
Then
\[
\mathbb E_{Q_{p^\star,q_t}}
[
Y_t(m)\mid\mathcal F_{t-1}
]
=
\sum_{i=1}^N q_t(i)
\left[
p_i^\star A_{t,i}(m)
+
(1-p_i^\star)B_{t,i}(m)
\right].
\]
Moreover,
\[
\mathrm{Var}
(
\zeta_t(m)\mid\mathcal F_{t-1}
)
=
\sum_{i=1}^N q_t(i)
\left[
p_i^\star A_{t,i}(m)^2
+
(1-p_i^\star)B_{t,i}(m)^2
\right]
-
\left(
\mathbb E_{Q_{p^\star,q_t}}
[
Y_t(m)\mid\mathcal F_{t-1}
]
\right)^2=\mathrm{Var}(Y_t(m)\,|\,\mathcal{F}_{t-1}).
\]
Thus, we have
\[
V_t
=
\sup_{m\in\mathcal M_K}
\sum_{s=1}^t
\mathrm{Var}
(
\zeta_s(m)\mid\mathcal F_{s-1}
)=\sup_{m\in\mathcal{M}_K}\sum_{s=1}^{t}\mathrm{Var}(Y_s(m)\,|\,\mathcal{F}_{t-1}),
\qquad
\sigma_t:=\sqrt{V_t}.
\]
By assumption, there exists $\vartheta\in(0,1/2)$ such that $\hat p_t(i),\ \hat p_{q_t}^{\dagger}(m; \hat p_{t-1})_i
\in[\vartheta,1-\vartheta]$
for all $t,i,m$, then with $L_\vartheta:=\log\frac{1-\vartheta}{\vartheta}$,
we have $|\zeta_t(m)|\le 2L_\vartheta$.
Therefore, one can take $b=2L_\vartheta$.

Then, by Lemma~\ref{lem:bound_for_general_CS_width}, with probability at least probability 1-$\delta$, for any $t\in[N]$,
\[
C_t^{\mathrm{RIPr}}\cap\mathcal{M}_k\subseteq C_t^{\mathrm{RIPr}}(\delta):=\left\{
m\in\mathcal M_K:
\sum_{s=1}^{t}
\mathbb{E}_{Q_{p^\star,q_s}}
\left[
\log E^{\mathrm{RIPr}}_s(m)\mid\mathcal F_{s-1}
\right]
-
\mathrm{pen}_t(\delta)
\le
\log\frac{1}{\alpha}
\right\},
\]
with $E^{\mathrm{RIPr}}_s(m)\equiv E^\mathrm{RIPr}_t(m, \hat p_{t-1}, q_t)$ and 
\[
\mathrm{pen}_t(\delta)
:=
2\max\left\{
2\sigma_t,\,
b\sqrt{\log\left(\frac{K\log N}{\delta}\right)}
\right\}
\sqrt{\log\left(\frac{K\log N}{\delta}\right)}=2\max\left\{2\sigma_t,2L_\vartheta\sqrt{\ell_{K,N,\delta}}\right\}\sqrt{\ell_{K,N,\delta}},
\]
where $\ell_{K,N,\delta}:=\log\left(\frac{K\log N}{\delta}\right)$.
Moreover, by Lemma~\ref{lem:oracle_ripr_curvature} and Lemma~\ref{lem:mismatch_in_ripr}, we have 
\begin{align}
\sum_{s=1}^{t}
\mathbb{E}_{Q_{p^\star,q_s}}
\left[
\log E^{\mathrm{RIPr}}_s(m)\mid\mathcal F_{s-1}
\right]&\geq\sum_{s=1}^{t}G_s(q_s, E_s^{\mathrm{RIPr}}(m, p^\star,q_s),p^\star)-D(Q_{p^\star,q_s}\|Q_{\hat p_{s-1}, q_s})\\
&\geq \sum_{s=1}^{t}\frac{2N^2}{\sum_{i=1}^{N}1/q_s(i)}(\theta^\star-m)^2-D(Q_{p^\star,q_s}\|Q_{\hat p_{s-1}, q_s}).
\end{align}
Therefore, by using $q_t(i)\geq h$ for some $h\in(0,1/N)$, we have at least probability $1-\delta$, for any $t\in[N]$,
\begin{align}
    |C_t^{\mathrm{RIPr}}(\delta)|&\leq2\sqrt{\frac{2\max\left\{2\sigma_t,2L_\vartheta\sqrt{\ell_{K,N,\delta}}\right\}\sqrt{\ell_{K,N,\delta}}+\sum_{s=1}^{t}D(Q_{p^\star,q_s}\|Q_{\hat p_{s-1}, q_s})+\log(1/\alpha)}{t2Nh}}.
\end{align}
Note that $hN\in(0,1)$ is a constant. Proof completes. 
\begin{remark}
    This bound may be too coarse since we have not controlled the $\sigma_t^{\mathrm{RIPr}}$ and the mismatch term. We will revise this issue in our updated version.
\end{remark}

\subsection{Proof of Theorem~\ref{thm:mismatch_impact_for_bet}}\label{proof:proof_of_mismatch_impact_for_bet}
We start by providing the formal statement of the theorem.
\begin{theorem}[(Formal) Approximated width of testing-by-betting approach]
For any $t\ge1$, suppose the querying rule satisfies $q_t(i)\geq h$ for some $h\in(0,1/N]$ for any $i\in[N]$, and the bet $\lambda_t(m)$ is chosen in $\frac{1}{2}\Lambda_h(m)$ for all $m\in\mathcal{M}_K$, and the optimized bet for \eqref{eq:exp_quadratic_lower_bound} is obtained and still in $\frac{1}{2}\Lambda_h(m)$. With probability at least probability $1-\delta$, for any $t\in[N]$,  
\begin{align}
    |C_t^{\mathrm{bet}}\cap \mathcal{M}_{K}|\leq8\sqrt{\frac{\log(1/\alpha)+\max\{\sigma^\mathrm{bet}_t,\sqrt{\ell_{K,N,\delta}}\}\sqrt{\ell_{K,N,\delta}}\left(2+\overline{\mathrm{MSE}}_t\right)}{tNh}}.
\end{align}
where $\mathcal{M}_K$ is the set of $K$ grids, $\sigma^\mathrm{bet}_t:=\sqrt{\sup_{m\in\mathcal{M}_K}\sum_{s=1}^{t}\mathrm{Var}(\log E_t^{\mathrm{bet}}(\cdot)\,|\,\mathcal{F}_{t-1})}$, $\ell_{K,N,\delta}:=\log\left(\frac{K\log N}{\delta}\right)$, and $\overline{\mathrm{MSE}}_t:=\frac{1}{t}\sum_{s=1}^{t}\frac{1}{N}\sum_{i=1}^{N}(p^\star(i)-\hat p_{s-1}(i))^2$.
\end{theorem}
Then, the proof is as follows.

Similar to the proof of Theorem~\ref{thm:mismatch_impact_for_RIPr}, we focus on deriving the $\sigma_t$ and $b$ under the betting-based CS.

We also define 
\[
\zeta_t(m):=Y_t(m)-\mathbb{E}_{Q_{p^\star,q_t}}[Y_t\,|\,\mathcal{F}_{t-1}] \qquad\forall\,t\ge1,
\]
where $Y_t(m)
:=
\log E_t^{\mathrm{bet}}(m)
=
\log\!\left(
1+\lambda_t(m)\bigl(\Phi_t(\hat p_{t-1},q_t)-m\bigr)
\right)$,
with predictable bets $\lambda_t(m)$ and a predictable querying rule $\{q_t\}_{t\ge1}$. 

Recall that $\Phi_t(\hat p_{t-1},q_t)
=
\frac{1}{N}\sum_{i=1}^{N}\hat p_{t-1}(i)
+
\frac{1}{N}\frac{Z_{I_t}-\hat p_{t-1}(I_t)}{q_t(I_t)}$. By $\hat p_{t-1}(i)\in[0,1]$ and $q_t(i)>h$ for all $i\in[N]$ and $t\ge1$, we know $-m-\frac{1}{Nh}\leq\Phi_t(\hat p_{t-1},q_t)-m\leq 1-m+\frac{1}{Nh}$ for all $t\ge1$.

Then, for all
$\lambda_t(m)\in \frac12\Lambda_h(m):=\left[\frac{-1/2}{1-m+1/(Nh)}, \frac{1/2}{m+1/(Nh)}\right]$, we have
\[
1+\lambda_t(m)\bigl(\Phi_t(\hat p_t)-m\bigr)\ge 1-\frac12=\frac12.
\]
Moreover, we have
\[
\left|
\lambda_t(m)\bigl(\Phi_t(\hat p_t)-m\bigr)
\right|
\le
\frac12(1+Nh)\leq1,
\]
where the last inequality is from $h\in(0,1/N)$.
Hence, $|Y_t(m)|\le\log2$. Then,
\begin{align}
|\zeta_t(m)|
&=
\left|
Y_t(m)
-
\mathbb E\!\left[Y_t(m)\mid\mathcal F_{t-1}\right]
\right| \\
&\le
|Y_t(m)|
+
\mathbb E\!\left[|Y_t(m)|\mid\mathcal F_{t-1}\right] \\
&\le
2\log 2.
\end{align}
Therefore, in Freedman's inequality, we take $b=2$.
Moreover, for each $m\in\mathcal M_K$,
\[
V_t(m)
:=
\sum_{s=1}^t
\mathrm{Var}(\zeta_s(m)\mid\mathcal F_{s-1})
=\sum_{s=1}^{t}\mathrm{Var}(\log E_t^{\mathrm{bet}}(\cdot)\,|\,\mathcal F_{t-1})\qquad V_t:=\sup_{m\in\mathcal M_{K}}V_t(m)
\]
We define $\sigma^{\mathrm{bet}}_t:=\sqrt{V_t}.$

Applying Lemma~\ref{lem:bound_for_general_CS_width} on the finite grid
$\mathcal M_K$, with probability at least $1-\delta$, for all
$t\le N$,
\[
C_t^{\mathrm{bet}}\cap\mathcal M_K
\subseteq
\left\{
m\in\mathcal M_K:
\sum_{s=1}^t \mathbb{E}_{Q_{p^\star,q_s}}[\log E^{\mathrm{bet}}_s(m)\,|\,\mathcal{F}_{s-1}]
-
\mathrm{pen}_t(\delta)
\le
\log\frac1\alpha
\right\},
\]
where, using $b=2$,
\[
\mathrm{pen}_t(\delta)
=
2\max\left\{
2\sigma_t,\,
2\sqrt{\ell_{K,N,\delta}}
\right\}
\sqrt{\ell_{K,N,\delta}}=4\max\left\{
\sigma_t,\,
\sqrt{\ell_{K,N,\delta}}
\right\}
\sqrt{\ell_{K,N,\delta}},
\]
where $\ell_{K,N,\delta}=\log\left(\frac{K\log N}{\delta}\right)$.

It remains to lower-bound the expected log-increment. Let
\[
X_t(m):=\Phi_t(\hat p_t)-m .
\]
We have known that, for all $m\in\mathcal M_K$,
\[
\lambda_t(m)X_t(m)\ge -\frac12 .
\]
Therefore, using the inequality $\log(1+x)\ge x-x^2$ for all $x\ge -1/2$, we obtain
\begin{align}
\mathbb{E}_{Q_{p^\star,q_t}}
\!\left[
\log E_t^{\mathrm{bet}}(\cdot)
\,\middle|\,
\mathcal F_{t-1}
\right] \ge
\lambda_t(m)
\mathbb{E}_{Q_{p^\star,q_t}}
\!\left[
X_t(m)\mid\mathcal F_{t-1}
\right]
-
\lambda_t^2(m)
\mathbb{E}_{Q_{p^\star,q_t}}
\!\left[
X_t^2(m)\mid\mathcal F_{t-1}
\right].
\end{align}
For the first term, since $\hat p_{t-1}$ and $q_t$ are $\mathcal{F}_{t-1}$-measurable, and
by the conditional unbiasedness, $\mathbb{E}_{Q_{p^\star,q_t}}
\!\left[
X_t(m)\mid\mathcal F_{t-1}
\right]
=
\theta^\star-m $. Then, we have 
\[
\mathbb{E}_{Q_{p^\star,q_t}}
\!\left[
\log E_t^{\mathrm{bet}}(\cdot)
\,\middle|\,
\mathcal F_{t-1}
\right] \geq \lambda_t(m)(\theta^\star-m)-\lambda_t^2(m)\mathbb{E}_{Q_{p^\star,q_t}}
\!\left[
X_t^2(m)\mid\mathcal F_{t-1}
\right].
\]

For the second moment, we have
\begin{align}
\mathbb{E}_{Q_{p^\star,q_t}}
\!\left[
X_t^2(m)\mid\mathcal F_{t-1}
\right] &=
(\theta^\star-m)^2
+
\mathrm{Var}_{Q_{p^\star,q_t}}
\left(
\Phi_t(\hat p_{t-1},q_t)\mid\mathcal F_{t-1}
\right)\\
&\le
(\theta^\star-m)^2
+
\frac1{N^2}
\sum_{i=1}^N
\frac{
p_i^\star(1-p_i^\star)
+
(p_i^\star-\hat p_{t-1}(i))^2
}{q_t(i)} .
\end{align}
By the assumption $q_t(i)>h$ for all $i\in[N]$ and $p_i^\star(1-p_i^\star)\le 1$,
\[
\begin{aligned}
\mathbb{E}_{Q_{p^\star,q_t}}
\!\left[
X_t^2(m)\mid\mathcal F_{t-1}
\right]
&\le
(\theta^\star-m)^2
+
\frac{1}{Nh}
+
\frac{1}{N^2h}
\sum_{i=1}^N
(p_i^\star-\hat p_{t-1}(i))^2 .
\end{aligned}
\]
Define $\mathrm{MSE}_t
:=
\frac1N\sum_{i=1}^N
(p_i^\star-\hat p_{t-1}(i))^2$.
Combining the above inequalities gives
\begin{align}
&\mathbb{E}_{Q_{p^\star,q_t}}
\left[
\log E_t^{\mathrm{bet}}(\cdot)
\,\middle|\,
\mathcal F_{t-1}
\right]\ge
\lambda_t(m)(\theta^\star-m)
-
\left(\lambda_t(m)\right)^2
\left[
(\theta^\star-m)^2
+
\frac{1}{Nh}
+
\frac{\mathrm{MSE}_t}{Nh}
\right].
\end{align}

This inequality holds for all $\lambda_t(m)\in\frac{1}{2}\Lambda_h(m)$. Then, since we consider using the optimized bet for \eqref{eq:exp_quadratic_lower_bound} and assume this bet is still in $\frac{1}{2}\Lambda_h(m)$ for each $m\in\mathcal M_{K}$, we have
\[
\mathbb{E}_{Q_{p^\star,q_t}}
\left[
\log E_t^{\mathrm{bet}}(\cdot)
\,\middle|\,
\mathcal F_{t-1}
\right]
\ge
\frac{
(\theta^\star-m)^2
}{
4\left[
(\theta^\star-m)^2
+
\frac{1}{Nh}
+
\frac{\mathrm{MSE}_t}{Nh}
\right]
}\geq \frac{
(\theta^\star-m)^2
}{
4\left[
1
+
\frac{1}{Nh}
+
\frac{\mathrm{MSE}_t}{Nh}
\right]
}.
\]
Let us further simplify this lower bound. By the fact that $f(x):=\left(1+\frac{1+x}{Nh}\right)^{-1}=\frac{Nh}{Nh+1+x}$ is a convex function for $x\ge0$, we have $\frac{1}{t}\sum_{s=1}^{t}f(x_s)\geq f(\frac{1}{t}\sum_{s=1}^{t}x_s)$ by Jensen's inequality. Therefore, 
\[
\frac{1}{t}\sum_{s=1}^{t}\mathbb{E}_{Q_{p^\star,q_s}}
\left[
\log E_s^{\mathrm{bet}}(\cdot)
\,\middle|\,
\mathcal F_{s-1}
\right]\ge\frac{1}{4}(\theta^\star-m)^2\frac{Nh}{Nh+1+\overline{\mathrm{MSE}}_t}\geq\frac{1}{4}(\theta^\star-m)^2\frac{Nh}{2+\overline{\mathrm{MSE}}_t},
\]
where $\overline{\mathrm{MSE}}_t:=\frac{1}{t}\sum_{s=1}^{t}\mathrm{MSE}_s$ and the last inequality is from $Nh\leq1$.

Therefore, by Lemma~\ref{lem:bound_for_general_CS_width} and $\sigma_t$ and $b$ defined in this proof, we have
\[
|C_t^{\mathrm{bet}}\cap \mathcal{M}_{K}|\leq8\sqrt{\frac{\log(1/\alpha)+\max\{\sigma^{\mathrm{bet}}_t,\sqrt{\ell_{K,N,\delta}}\}\sqrt{\ell_{K,N,\delta}}\left(2+\overline{\mathrm{MSE}}_t\right)}{tNh}}.
\]
Proof completes.
\begin{remark}
    Again, this bound may be too coarse since we have not controlled the $\sigma_t^{\mathrm{bet}}$. We will revise this issue in our updated version.
\end{remark}

\end{document}